\documentclass{article}

\usepackage{arxiv}

\usepackage[utf8]{inputenc} 
\usepackage[T1]{fontenc}    
\usepackage{hyperref}       
\usepackage{url}            
\usepackage{booktabs}       
\usepackage{amsfonts}       
\usepackage{nicefrac}       
\usepackage{microtype}      
\usepackage{lipsum}
\usepackage{graphicx}
\graphicspath{ {./images/} }

\usepackage{wrapfig}

\author{
 Nil Ayday \\
  School of Computation, Information and Technology\\
Technical University of Munich\\
\texttt{nil.ayday@tum.de}
   \And
Mahalakshmi Sabanayagam\\
 	School of Computation, Information and Technology\\
Technical University of Munich
  \And
 Debarghya Ghoshdastidar \\
  School of Computation, Information and Technology\\
Technical University of Munich
}

\usepackage{amsmath}
\usepackage{amssymb}
\usepackage{mathtools}
\usepackage{amsthm}

\usepackage[capitalize,noabbrev]{cleveref}

\theoremstyle{plain}
\newtheorem{theorem}{Theorem}[section]

\newtheorem{lemma}[theorem]{Lemma}
\newtheorem{corollary}[theorem]{Corollary}
\theoremstyle{definition}
\newtheorem{definition}[theorem]{Definition}
\newtheorem{assumption}[theorem]{Assumption}
\theoremstyle{remark}
\newtheorem{remark}[theorem]{Remark}

\usepackage[textsize=tiny]{todonotes}

\usepackage[utf8]{inputenc} 
\usepackage[T1]{fontenc}    
\usepackage{hyperref}       
\usepackage{url}            
\usepackage{booktabs}       
\usepackage{amsfonts}       
\usepackage{nicefrac}       

\usepackage{lipsum}
\usepackage{graphicx}
\graphicspath{ {./images/} }

\usepackage{subcaption}

\usepackage{graphicx}
\usepackage{subcaption}
\usepackage{amsthm}
\usepackage{amsmath}

\usepackage{etoolbox} 
\usepackage{array} 

\usepackage[utf8]{inputenc} 
\usepackage[T1]{fontenc}    
\usepackage{hyperref}       
\usepackage{url}            
\usepackage{booktabs}       
\usepackage{amsfonts}       

\usepackage{xcolor}         

\usepackage{natbib} 
\usepackage{comment}
\usepackage{tcolorbox}
\definecolor{hellblue}{RGB}{220, 230, 241}
\tcbset{
    colback=hellblue,
    colframe=hellblue,
    arc=2mm,
    boxrule=0pt,
}

\newenvironment{mybox}[2]{%
    \begin{tcolorbox}[colback=#1,colframe=#1]
        {\large \textbf{#2}}\\[4pt]
}{%
    \end{tcolorbox}
}
\usepackage{mdframed}

\title{Exact Generalisation Error Exposes Benchmarks Skew Graph Neural Networks Success (or Failure)}

\begin{document}
\maketitle

\begin{abstract}
Graph Neural Networks (GNNs) have become the standard method for learning from networks across fields ranging from biology to social systems, yet a principled understanding of what enables them to extract meaningful representations, or why performance varies drastically between similar models, remains elusive. These questions can be answered through the generalisation error, which measures the discrepancy between a model’s predictions and the true values it is meant to recover. Although several works have derived generalisation error bounds, learning theoretical bounds are typically loose, restricted to a single architecture, and offer limited insight into what governs generalisation in practice. In this work, we take a fundamentally different approach by deriving the \emph{exact generalisation error} for a broad range of linear GNNs, including convolutional, PageRank-based, and attention-based models, through the lens of signal processing. Our exact generalisation error exposes a strong benchmark bias in existing literature: commonly used datasets exhibit high alignment between node features and the graph structure, inherently favouring architectures that rely on it. We further show that the similarity between connected nodes (homophily) decisively governs which architectures are best suited for a given graph, thereby explaining how specific benchmark properties systematically shape the reported performance in the literature. Together, these results explain when and why GNNs can effectively leverage structure and feature information, supporting the reliable application of GNNs.
\end{abstract}

\section{Introduction}
Graph Neural Networks (GNNs) have emerged as a central paradigm, achieving state-of-the-art results on graph-structured data across various tasks ranging from molecular property prediction in biology \citep{buterez2024transfer}, to analysing social interactions in social sciences \citep{DBLP:conf/www/Fan0LHZTY19}, enhancing recommendation systems in data science \citep{wu2022graph}, and understanding traffic flow in transportation networks \citep{DBLP:journals/corr/abs-2401-00713}. 
Despite widespread application across scientific domains and progress in developing a suite of GNNs, the statistical understanding of what enables GNNs to extract meaningful representations and what limits their performance remains largely incomplete.

Most current insights into GNN behaviour come from empirical comparisons on benchmark datasets, often with strong and unexamined inductive biases such as high homophily, where edges predominantly connect similar nodes, and specific feature-structure correlations \citep{bechler-speicher2025position}. 
Recent findings show that model rankings can be highly sensitive to dataset splits and training protocol, and that simple architectures can outperform more sophisticated models when evaluated fairly \citep{schur2018pitfalls}. However, in practice, conclusions about which architecture works best are frequently driven by these benchmarks rather than by a principled understanding of generalisation. For example, Graph Convolutional Network (GCN) \citep{DBLP:conf/iclr/KipfW17} is known to perform well on highly homophilic Cora citation network \citep{10.5555/295240.295725} but struggles on heterophilic graphs such as Squirrel 
\citep{DBLP:journals/compnet/RozemberczkiAS21}
 (Figure \ref{fig:intro}). This observation has led to the prevailing intuition that GCNs excel on homophilic graphs but falter on heterophilic ones \cite{DBLP:conf/nips/ZhuYZHAK20}. 

To illustrate the limitations of benchmark-driven intuition, we design two synthetic experiments, 
presented in Figure \ref{fig:intro}. In the first, we decrease the homophily score of Squirrel, measured as the ratio of edges connecting nodes within class \citep{DBLP:conf/nips/ZhuYZHAK20}, by adding edges between nodes of different classes. Contrary to the prevailing belief that GCNs deteriorate as homophily decreases, GCN achieves \emph{higher} accuracy on the more heterophilic version of Squirrel (Squirrel ($h\downarrow$) in Figure \ref{fig:intro}). But does GCN in fact reliably perform well on homophilic graphs? To verify this, in the second experiment, we push homophily to its extreme by constructing a perfectly homophilic graph with two classes: two disjoint copies of the Cora network, each corresponding to a single class. When the copies shared identical features, GCN performs \emph{only as good as random guess} 
despite a perfect homophily score of 1 (Copied Features in Figure \ref{fig:intro}). 
These counterexamples demonstrate that GNN performance is not determined by homophily alone but also by the interaction between graph structure and node features. If the behaviour of GCN, the simplest and most widely studied GNN, cannot be reliably understood from standard benchmarks, then experiments alone is insufficient, highlighting a need for a theoretical lens.

Theoretical works have attempted to shed light on GNNs from multiple perspectives, including signal processing, and generalisation error bounds. From a signal processing perspective, 
GNNs can be considered as \emph{graph filters} \citep{DBLP:journals/corr/abs-1905-09550} that operate on signals defined over the nodes of a graph.
The graph filter view reveals what kind of filter a GNN can represent. For example, GCNs \citep{DBLP:conf/iclr/KipfW17} are known to behave as low-pass filters \citep{DBLP:journals/corr/abs-1905-09550}, suppressing high-frequency components. However, as shown in Figure \ref{fig:intro} Squirrel ($h\downarrow$), GCNs can still perform well on graphs with strong heterophily, where dissimilarity produces high-frequency signals, challenging the view that their success is solely due to low-pass filtering. 
On the other hand, generalisation error measures a model's predictive performance on data with unknown targets (labels) and deriving it enables us to understand the conditions under which a GNN is expected to succeed or fail. While empirical results can demonstrate that a model works in practice on a particular dataset, they often cannot explain the underlying reasons for good performance and are susceptible to biases in benchmarks. 

Several prior works have upper bounded the generalisation error of GNNs using tools from statistical learning theory, including approaches based on algorithmic stability \citep{DBLP:conf/kdd/VermaZ19,DBLP:journals/ijon/ZhouW21},  Rademacher complexity \citep{DBLP:conf/nips/EsserVG21,DBLP:journals/corr/abs-2102-10234}, and PAC-Bayes theory \citep{DBLP:conf/iclr/LiaoUZ21}. However, these bounds are often restricted to a single architecture, tend to be loose, and do not capture the true generalisation. 
Moreover, deriving bounds for GNNs is challenging due to the dependencies in the graph data, unlike traditional machine learning settings with independent data samples \citep{DBLP:journals/corr/abs-1906-11300}. We therefore ask: 

\emph{Is it possible to derive exact generalisation error for GNNs?
Consequently, can it reveal how the graph data---features, structure, and targets---, and the graph filter contribute to performance? Can theory help to assess whether standard benchmarks reliably capture these influences?}

\begin{figure}
\centering
\scalebox{0.72}{ 
\begin{tikzpicture}[
    box/.style={rectangle, draw, rounded corners, minimum width=3.5cm, minimum height=1.2cm, align=center},
    arrow/.style={-Stealth, thick}
]

\node[box] (cora) at (-1,4) {
    \textbf{Cora (original)}\\
    \#classes = 7\\
    \fcolorbox{black}{gray!20}{\parbox{3cm}{\centering $h = 0.81$\\ Acc = 78.67 $\pm$ 2.3}}
};

\node[align=center] at (-1,0.75) { 
  \textbf{Duplicated Cora}\\[2pt]
  Two disjoint copies of Cora\\[2pt]
  Each copy is assigned a distinct label.\\[4pt] 
  Two feature conditions are considered:\\[2pt]
  1. \textbf{Random features}:\\ Both copies have random features.\\
  2. \textbf{Copied features}:\\ Both copies have node features of Cora.
};

\node[box] (cora_left) at (-3,-2.5) {
    \textbf{Random Features}\\
    \#classes = 2\\
    \fcolorbox{black}{gray!20}{\parbox{3cm}{\centering $h = 1.0$\\ Acc = 95.35 $\pm$ 1.1}}
};

\node[box] (cora_right) at (1,-2.5) {
    \textbf{Copied Features}\\
    \#classes = 2\\
    \fcolorbox{black}{gray!20}{\parbox{3cm}{\centering $h = 1.0$\\ Acc = 49.04 $\pm$ 1.5}}
};

\draw[->] (cora.south west) to[out=200,in=90] (cora_left.north west);
\draw[->] (cora.south east) to[out=-20,in=90] (cora_right.north east);

\node[box] (squirrel) at (5, 4) {
    \textbf{Squirrel (original) }\\
    \#classes = 5\\
    \fcolorbox{black}{gray!20}{\parbox{3cm}{\centering $h = 0.22$\\ Acc = 33.85 $\pm$ 1.7}}
};

\node[box] (squirrel_child) at (5, -2.5) {
    \textbf{Squirrel ($h\downarrow$)}\\
    \#classes = 5\\
    \fcolorbox{black}{gray!20}{\parbox{3cm}{\centering $h = 0.13$\\ Acc = 68.57 $\pm$ 1.5}}
};

\draw[->] (squirrel.south west) -- node[right, align=left] {
    Squirrel modified \\
    with edges between \\
    different-class nodes, \\
    increasing heterophily
} (squirrel_child.north west);

\end{tikzpicture}
}
\caption{
\textbf{GCN performance} on two common benchmark (Cora \citep{10.5555/295240.295725} and Squirrel \citep{DBLP:journals/compnet/RozemberczkiAS21}). Experiments using synthetic variants of the benchmarks demonstrate that \emph{homophily alone does not explain performance} and benchmarks do not fully capture GNN behavior.
\textbf{$h$ = homophily score}   ($h=1.0$ indicates perfect homophily); \textbf{Acc =test accuracy.}}
\label{fig:intro}
\end{figure}

\textbf{Technical Contributions.}
We answer all three questions positively by developing a novel theoretical framework from a signal-processing perspective and deriving the exact generalisation error for GNNs under this framework (Section \ref{Ch:gen_err}). Unlike previous works, which are typically restricted to specific architectures and are often unable to handle attention-based models, our framework provides a unified treatment of a broad range of GNNs including convolutional models (GCN \citep{DBLP:conf/iclr/KipfW17}, GIN \citep{DBLP:conf/iclr/XuHLJ19}, GraphSAGE \citep{DBLP:conf/nips/HamiltonYL17}, ChebNet \citep{DBLP:conf/nips/DefferrardBV16}, CayleyNet \citep{DBLP:journals/tsp/LevieMBB19}, Highpass \citep{DBLP:conf/acml/ZhangL23}, FAGCN \citep{DBLP:conf/aaai/BoWSS21}), PageRank-based models (PPNP \citep{DBLP:conf/iclr/KlicperaBG19}, GPR-GNN \citep{DBLP:conf/iclr/ChienP0M21}), and attention-based models ( graph attention network (GAT) \citep{DBLP:conf/iclr/VelickovicCCRLB18} and Specformer \citep{DBLP:conf/iclr/BoSWL23}). To derive the first exact generalisation error across these models, our theoretical analysis focuses on linear GNNs, while \textbf{allowing non-linearity in the graph filters}.
While this choice may seem restrictive, both empirical and theoretical studies have shown that linear GNNs achieve competitive performance in semi-supervised node classification, thereby validating the practical relevance of our analysis \citep{DBLP:conf/icml/WuSZFYW19, DBLP:journals/tmlr/SabanayagamEG23}.
\vspace{-0.36mm} We also conduct experiments with non-linear GNNs, and the empirical trends match our theory well, indicating that the insights derived from linear models remain relevant.

\textbf{Insights from Theory.}
Section \ref{Ch:insights} uses our exact characterisation of generalisation error to provide insights into the limitations of GNNs and reveals how benchmark design skews the perceived success.
$(i)$ In Section \ref{Ch:mislignment}, we investigate how the alignment between the graph structure and node features influences generalisation. Our analysis shows that when the graph and features are misaligned, GNNs struggle to combine these sources of information: only aligned information contributes to prediction. This explains the surprising behaviour observed on the synthetic Cora dataset in Figure \ref{fig:intro}, where GCNs perform well with random features, which do not carry any information, but fail when features are copied, as the copied features are misaligned with the new graph. But how can we determine misalignment on a real dataset? The \textbf{misalignment score} (Definition \ref{def:Misalignment}), derived from the generalisation error, can be computed for any dataset and model, providing valuable information for \emph{model selection} by revealing the susceptibility of multiplicative architectures, such as GCN, to misalignment. Consequently, the misalignment score serves as a diagnostic tool to evaluate dataset biases toward certain architectures and to validate both existing and newly proposed benchmarks, making this bias measurable within a principled framework. $(ii)$ Section \ref{Ch:hetero} provides an analysis of the \emph{role of homophily}, where we quantify the effect of homophily on generalisation error. This allows us to assess how GNN performance is influenced by the prevalence of homophily in commonly used benchmark datasets, thereby revealing the limitations of different GNNs, including Chebyshev networks, pagerank-based networks, and a GCN. 
Moreover, our analysis exposes a surprising behaviour: the theoretical generalisation error of GCN recovers under extreme heterophily, thereby explaining the observation in Figure \ref{fig:intro}, where adding more heterophilic edges to Squirrel improves performance. $(iii)$ In Section \ref{Ch:GAT}, we study graph attention and discuss how Specformer overcomes a key limitation of GAT by allowing different frequency responses for repeating eigenvalues. To formalise this, we incorporate the underlying principles of Specformer and GAT into our framework (Definition \ref{Def:gat}), providing a suitable proxy for the theoretical analysis of attention-based architectures without relying on their practical implementations. Experiments with the original non-linear architectures support the relevance of this analysis. These theoretical insights explain when GNNs can effectively use both the graph and the features—--and more importantly, when they fail. Critically, our results expose that benchmarks often mask the limitations of GNNs. 

\section{Preliminaries}
\label{Ch:pre}
This section introduces the notation, learning setup, and the GNN models used in our analysis. We also describe the statistical framework and define the generalisation error.

\textbf{Notation.} Matrices are denoted by uppercase letters (\( A \)), while vectors are denoted by lowercase letters (\( a \)) with its entries $a_i$. For a matrix \( A \), \( A^{\dagger} \) denotes its Moore-Penrose pseudoinverse. The identity matrix is \( I \). \( e_i \) is the \( i \)-th standard basis vector. \( \langle \cdot, \cdot \rangle \) denotes the standard scalar product, and \( |\cdot| \) is the cardinality of a set.

\textbf{Graph Data.} Let \(G\) be a graph with \( n \) nodes and an arbitrary number of edges. The \(d\)-dimensional features for all \(n\) nodes are collected in rows of the \(n \times d\) matrix \(Z\). The structure information of the graph is given by the \(n \times n\) adjacency matrix $A$. The symmetric normalized Laplacian is defined as \( L = I - D^{-1/2} A D^{-1/2} \), where \( D \in \mathbb{R}^{n \times n} \) is the diagonal degree matrix and \( I \) is the identity matrix. Through eigendecomposition, we write \(\displaystyle L = U \Lambda U^\top \), where each column of \( U \in \mathbb{R}^{n \times n} \) is an eigenvector of \( L \), \( \Lambda \in \mathbb{R}^{n \times n} \) is the diagonal eigenvalue matrix of \( L \) with $\Lambda_{ii} =: \lambda_i$.

\textbf{Learning Setup.} We consider a semi-supervised learning problem, where the model has access to the Laplacian matrix \( L \in \mathbb{R}^{n \times n} \) and node features \( Z \in \mathbb{R}^{n \times d} \). Let \( V\) denote the set of all nodes, which is partitioned into a training set \( V_{\text{train}} \), with \(|V_{\text{train}}|=n_{\text{train}}\) and a test set \( V_{\text{test}} \). Each node is associated with a target (label) \(y_{i}\) which is only observed for nodes in \(V_{\text{train}}\) and the task is to predict the  \(n-n_{\text{train}}\) 
 targets (labels) for nodes in \( V_{\text{test}} \).
 
\textbf{GNNs as graph filters.} GNNs are a class of deep learning models designed to learn from graph-structured data. A core component of GNNs is the graph convolution operation \citep{DBLP:journals/tnn/WuPCLZY21}. A graph convolution layer transforms the initial node features \(Z \), into a new representation 
\(
 H :=SZ,
\)
where \( S \in \mathbb{R}^{n \times n} \) is the graph convolution that determines how the node features are propagated to neighboring nodes. The output of the GNN is obtained by applying a learnable vector \( \hat{\theta} \in \mathbb{R}^d \) to \(H\) via a linear transformation, resulting in \( H \hat{\theta} \). For GNNs with multiple convolution layers, the graph convolution is applied repeatedly to propagate information across larger neighborhoods. Specifically, for a linear GNN with \( \ell \) layers, the hidden representation becomes
\(
H := S^\ell Z.
\)
The graph filter \(S\) can be defined based on spatial or spectral information of the graph. However, various GNNs, both spectral and spatial, can be expressed in the spectral domain of \(L\) \citep{DBLP:conf/iclr/BalcilarRHGAH21} as 
\[
S = U g(\Lambda) U^\top,
\]
where \(\Lambda = \operatorname{diag}(\lambda_1, \dots, \lambda_n)\) is the eigenvalue matrix of \(L\), and \( g(\Lambda) \) is a spectral function applied to its diagonal entries. \( g(\Lambda) \) determines how different frequency components (corresponding to the eigenvalues of the  graph Laplacian) are weighted during the convolution, and is referred to as the \emph{frequency response}. Thus, this formulation allows the graph convolutions to be represented by their frequency response, where $g(\Lambda)$ depends on the GNN considered. While \cite{DBLP:conf/iclr/BalcilarRHGAH21} derive frequency responses for standard architectures such as GCN, GIN, CayleyNet and ChebNet, we extend this analysis to additional models including PPNP, GPR-GNN, Highpass, FAGCN, GAT and Specformer. Representative examples of \(g(\lambda)\) for selected GNNs are presented in Table~\ref{tab:gnn_filters2}, with further derivations—--both from prior work and our own contributions--—provided in Appendix \ref{Ap:frequency_response}.
\begin{table}[]
\centering
\caption{Frequency responses \( g(\lambda_{i})\) of selected GNNs.}
\begin{tabular}{@{}l l@{}}
\toprule
\textbf{GNN} & \textbf{Frequency response (\( \bar{p} \) = average degree)} \\
\midrule
GCN & \( g(\lambda_{i}) = 2\left(1 - \frac{\lambda_{i}}{2}\right) \) \\
Highpass & \( g(\lambda_{i}) = \lambda_{i} \) \\
PPNP & \( g(\lambda_{i}) \approx \alpha \left(1 - (1 - \alpha)\left(1 - \frac{\lambda_{i} \bar{p}}{\bar{p} + 1}\right)\right)^{-1} \) \\
GPR-GNN & \( g_{\gamma,K}(\lambda_{i}) \approx \sum_{k=0}^{K} \gamma_k \left(1 - \frac{\lambda_{i} \bar{p}}{\bar{p} + 1} \right)^k \) \\
\bottomrule
\end{tabular}
\label{tab:gnn_filters2}
\end{table}

\textbf{Statistical Framework.} We consider a probabilistic model for the data, with a fixed design setting \citep{bach2024learning}, where a deterministic unobserved \( X \in \mathbb{R}^{n \times d} \) denote the latent features of data. The true target is linear in \( X \) with additive noise, as given by:
\[
y = X\theta^{\star} + \epsilon, \hspace{2mm} \mathbb{E}[\epsilon]=0   \hspace{2mm}, \mathbb{E}[\epsilon \epsilon^T] = \sigma^2 I
\]
where \( \theta^{\star} \in \mathbb{R}^d \) is the optimal parameter vector with a random prior, the noise term \( \epsilon \) has mean zero and isotropic variance. 
We assume that the observed node features \( Z \) and the latent features \(X\) lie in the eigenspace of \(L\), given by:
\begin{equation}
\frac{1}{\sqrt{n}}X = U \, {\Lambda^\star}^{1/2} \Phi^\top, \quad \frac{1}{\sqrt{n}} Z = U \, \Lambda_f^{1/2} \Phi^\top
\label{eq:XZ}
\end{equation}
where \( U \in \mathbb{R}^{n \times n}\)  corresponds to the eigenvectors of \( L \), respectively and  \( \Phi^\top \in \mathbb{R}^{n \times d} \) is an orthogonal matrix  such that \(
\Phi^\top \Phi =
\begin{bmatrix}
I_d & 0 \\
0 & 0
\end{bmatrix}_{n\times n},\) where \(I_d\) is the \(d \times d\) identity matrix and the remaining entries are zero. Because of the choice of $\Phi$, the feature covariance matrices have $d$ nonzero eigenvalues, which remain constant in $n$ under the $1/\sqrt{n}$ scaling. 
This framework allows quantifying the impact of graph and feature information on the performance (Section \ref{Ch:mislignment}) and the modeling of different kind of data (Section \ref{Ch:hetero}). Within this setup, we can express the node representations  \( H \) of a linear GNN with \( \ell\) layers:
\begin{equation}
\frac{1}{\sqrt{n}}H=  U g(\Lambda)^{ \ell}{\Lambda_f}^{1/2} \Phi^T=:U \tilde{\Lambda}^{\frac{1}{2}} \Phi^T,
\label{Eq:H}
\end{equation}
where we define \(\tilde{\Lambda} := \operatorname{diag}(\tilde{\lambda}_1, \dots, \tilde{\lambda}_n)\), with \(\tilde{\lambda}_i := g(\lambda_i)^{2\ell} ({\Lambda_f})_i\). We study estimator \(\hat{\theta}\) of the form
\begin{align*}
&\hat{\theta} = \underset{\theta \in \mathbb{R}^d}{\text{argmin}} \sum_{i \in V_{\text{train}}} \left( (H\theta)_i - y_i \right)^2 + \sigma^2 \|\theta\|^2 \\
&= \left(H_{\text{train}}^\top H_{\text{train}} + \sigma^2 I\right)^{\dagger} H_{\text{train}}^\top \mathbf{y},
\end{align*}
where \( H_{\text{train}} \in \mathbb{R}^{n \times d} \) is the representation of the training nodes, on which \(\hat{\theta}\) is fitted, as the optimally tuned parameter corresponding to the minimum-norm solution.

\textbf{Generalisation Error.}
The generalisation error quantifies the performance of the model by comparing the predicted outputs with the true targets. The generalisation error of a GNN with \(H\) in the above fixed design setting, where the data \(S, Z, X\) are deterministic, is given by:

\begin{equation}
 R_{H} = \mathbb{E}_{V,\epsilon,\theta^\star}\left[ \frac{1}{n} \sum_{i \in V} \left((H \hat{\theta})_{i} -  (X \theta^\star)_{i} \right)^2 \right],
\label{def:gen_err}
\end{equation}
 where expectation is over the uniformly
random test–train split \(V=V_{\text{train}}\cup V_{\text{test}}\), optimal parameter \(\theta^\star\) and noise \( \epsilon\).

\section{Generalisation Error in Fixed Design}
\label{Ch:gen_err}
Under our statistical framework introduced in Section~\ref{Ch:pre}, where node features and GNN representations are spectral functions of \(L\), the generalisation error can be \emph{exactly} characterised by the interaction between the spectrums of the node representation \(\tilde{\Lambda}\) (Equation~\ref{Eq:H}) and the underlying data matrix \(\Lambda^\star\) (Equation~\ref{eq:XZ}). We assume the following.

\begin{assumption}[\textbf{Population Covariance Matrix}]
We assume that the empirical covariance of the full dataset matches that of the training set,
\(
\frac{1}{n_{\text{train}}} X_{\text{train}}^\top X_{\text{train}} =\frac{1}{n} X^\top X.
\)
 \label{As:general}
\end{assumption}
Assumption \ref{As:general} simplifies the analysis by ensuring the training set covariance matches that of the full dataset, allowing an exact characterisation of the generalisation error. In Appendix \ref{Ap:covariance}, we motivate this assumption by showing that it holds in expectation and provide bounds on the deviation of generalisation error due to this assumption, which vanishes in the large
$n$ limit with high probability.
\begin{assumption}[\textbf{Parameter Prior}]
Let \( \{\phi_i\}_{1,..,n} \) be the columns of \( \Phi \) from Equation~\ref{eq:XZ}. We consider that the optimal parameter \( \theta^\star \) satisfies a prior such that, for \( i \neq j \),  
    \(\displaystyle
    \mathbb{E} [\langle \phi_i, \theta^\star \rangle \langle \phi_j, \theta^\star \rangle] = 0.
    \)
\label{As:aniso_prior}
\end{assumption}
Assumption \ref{As:aniso_prior} is mild given orthogonal $\Phi$ and is common in regression analysis \citep{DBLP:journals/focm/CaponnettoV07,DBLP:journals/corr/abs-2406-08466}.
Under these assumptions, the generalisation error can be expressed in terms of \( \tilde{\Lambda} \), \( \Lambda^{\star} \), and the prior structure of \( \theta^\star \), and the abbreviation \( c := \tfrac{ \sigma^2}{n_{\text{train}}} \), leading to the following theorem, proved in Appendix \ref{Ap:theo_proof}.
\begin{theorem}[\textbf{Generalisation Error of GNNs}]
Under the framework in Section \ref{Ch:pre} and Assumptions \ref{As:general} and \ref{As:aniso_prior}, the generalisation error $R_{H}$, as defined in Eq.~\eqref{def:gen_err}, is given by:
    \begin{align*}
        R_{H} 
        &= \sum_{i = 1}^{d} \left( \frac{\tilde{\lambda}_{i} c}{\tilde{\lambda}_{i} + c} - \left(\tilde{\lambda}_{i} - \lambda^{\star}_{i}\mathbb{E} \langle \phi_i, \theta^\star \rangle^2\right) \frac{c^2}{\left(\tilde{\lambda}_{i} + c\right)^2} \right),
    \end{align*}
    where \(\tilde{\lambda}_{i}=g(\lambda_i)^{2\ell} (\Lambda_f)_i\) (Eq.~\ref{Eq:H}), and \(\lambda^{\star}_{i}=(\Lambda^\star)_i\) (Eq.~\ref{eq:XZ}), with \(g(\Lambda)\), \(\Lambda_f\), and \(\Lambda^\star\) governing the influence of the graph filter, the observed features, and the latent features.
    \label{thm:aniso}
\end{theorem}

\section{Insights via Generalisation Error: \\Under What Conditions Do GNNs Fail?}
\label{Ch:insights}
Using the generalisation error of GNNs in Theorem \ref{thm:aniso}, we derive several new insights in this section. $(i)$ \emph{Failure due to feature-graph-target misalignment:} Most GNNs rely on multiplicative graph convolutions \citep{DBLP:journals/tnn/WuPCLZY21} which has been identified as the cause of oversmoothing as the networks get deeper \citep{DBLP:conf/nips/Keriven22}. While this offers explanation of the failure of deeper GNNs, we theoretically show even one-layer GNNs can fail when there is misalignment between graph, feature and target in Section \ref{Ch:mislignment}. 
$(ii)$ \emph{Sensitivity to heterophilic data:} 
Most popular models, such as GCN, have been reported to degrade significantly on heterophilic graphs \citep{DBLP:conf/nips/ZhuYZHAK20}.
We theoretically quantify the sensitivity of several GNNs to heterophily in Section \ref{Ch:hetero} by connecting it to the spectral properties of the GNN filters. 
$(iii)$ \emph{Limitations of attention-based models:} Finally, we investigate attention-based models (GAT, Specformer) by considering their proxies with trainable convolution in Section \ref{Ch:GAT}, and show that GAT, while offering adaptive frequency response, still suffer from limitations when eigenvalues of the Laplacian are repeated. Implementation details of the experiments are provided in Appendix \ref{Ap:Experiments}.

For the analysis in Sections \ref{Ch:mislignment} and \ref{Ch:hetero}, we simplify Assumption \ref{As:aniso_prior} to isotropic prior, i.e., the optimal $\theta^\star$ is isotropic. We note that this is only for ease of exposition.

\begin{corollary}[\textbf{Generalisation Error with Isotropic Parameter Prior}]
\label{gen-err}
Under the framework in Section \ref{Ch:pre} and Assumptions \ref{As:general} and isotropic parameter prior on the optimal parameter vector, i.e., \( \mathbb{E}[\theta^{\star} \theta^{\star\top}] = I \), the generalisation error of a GNN with representation \(H\) is given by:
\[
R_H = 
\sum_{i:\,\tilde{\lambda}_i=0} \lambda_i^\star
+ \sum_{i:\,\tilde{\lambda}_i>0} \left[
   \frac{c\tilde{\lambda}_i}{\tilde{\lambda}_i + c}
   - (\tilde{\lambda}_i - \lambda_i^\star)
     \frac{c^2}{(\tilde{\lambda}_i + c)^2}
\right].
\]
 with the minimum of \(R_{H}\) achieved when \(\tilde{\lambda}_i = \lambda^\star_i\) for all \(i\).
\end{corollary}

Since the generalisation error in Corollary~\ref{gen-err} is minimized when $\tilde{\lambda}_i = \lambda^\star_i$, the first part of the second term,
\(\sum_{\substack{i: \tilde{\lambda}_{i}>0}} \frac{\tilde{\lambda}_{i} c}{\tilde{\lambda}_{i} + c},
\) is the irreducible part of the error, while the second part,
\(
\sum_{\substack{i: \tilde{\lambda}_{i}>0}} \left( \lambda_i^\star-\tilde{\lambda}_{i}\right) \frac{c^2}{(\tilde{\lambda}_{i} + c)^2},
\)
is small if \( \tilde{\lambda}_i \) is close to \( \lambda^\star_i \).
Consequently, the success of a GNN depends on the interaction between the graph filter (determined by the GNN), the graph, and node features. In the following sections, we examine the conditions under which GNNs fail to understand their working mechanisms and limitations.

\subsection{Misalignment: When Naive Concatenation Outperforms Convolution}
\label{Ch:mislignment}
The idea behind GNNs is to exchange feature information between connected nodes to use their dependencies \citep{khemani2024review}. However, our analysis suggests that GNNs cannot always fully take advantage of graph information and node features, when the adjacency matrix (\( A \)), node features (\( Z \)), and the target (\( y \)) do not align. To quantify the misalignment, we introduce the following definition:

\begin{definition}[\textbf{Misalignment of GNN}]
For any node representation \( H \) produced by a GNN and underlying data matrix \( X \),  
the misalignment score is defined as:  \(
d(H, X) =  \frac{1}{n} \operatorname{Tr} \big( (I - P_H)X X^T \big),\)
where \( P_H = H(H^TH)^{\dagger}H^T\) is the projection matrix onto the column space of \( H \).
\label{def:Misalignment}
\end{definition}

The misalignment score measures the extent to which the target space $X$ is not represented in the GNN embedding $H$.  
When the subspaces of \( H \) and \( X\) coincide, the misalignment is zero, indicating perfect alignment. Misalignment score reaches its maximum when the subspaces are orthogonal.

\begin{lemma}[\textbf{Misalignment as a Component of Generalisation Error}]
\label{lemma:misalignment}
Under the framework in Section \ref{Ch:pre}, where \( X \) and the node representation \( H \) lie in the same subspace, the misalignment defined in Definition~\ref{def:Misalignment} equals \(
d(H, X) =  \sum_{\substack{i: \hspace{0.5mm}\tilde{\lambda}_i = 0}} \lambda_i^\star
\), the first term in the generalisation error expression in Corollary~\ref{gen-err}.
\end{lemma}

To understand how misalignment influences generalisation error, we consider a GCN with the symmetric normalised adjacency matrix as the convolution matrix. The
subspace spanned by \( H \) becomes smaller when the graph \( A \) and the features \( Z \) are not aligned \citep{klepper2023relatinggraphautoencoderslinear} due to the multiplicative nature of the convolution operation. The features or graph can even hurt generalisation if they are not aligned by increasing $R_H$. An example of this is already presented in Figure 1. On duplicated Cora with copied features: despite perfect homophily, GCN fails completely, highlighting that structure-feature alignment can dominate $R_H$. In such scenarios, where the graph and feature matrix are misaligned, even a simple concatenation of the adjacency and feature matrices \citep{chen2022demystifyinggraphconvolutionsimple} can outperform the convolution operation. 
\begin{wrapfigure}{r}{0.48\columnwidth}
    \centering
    \includegraphics[width=\linewidth]{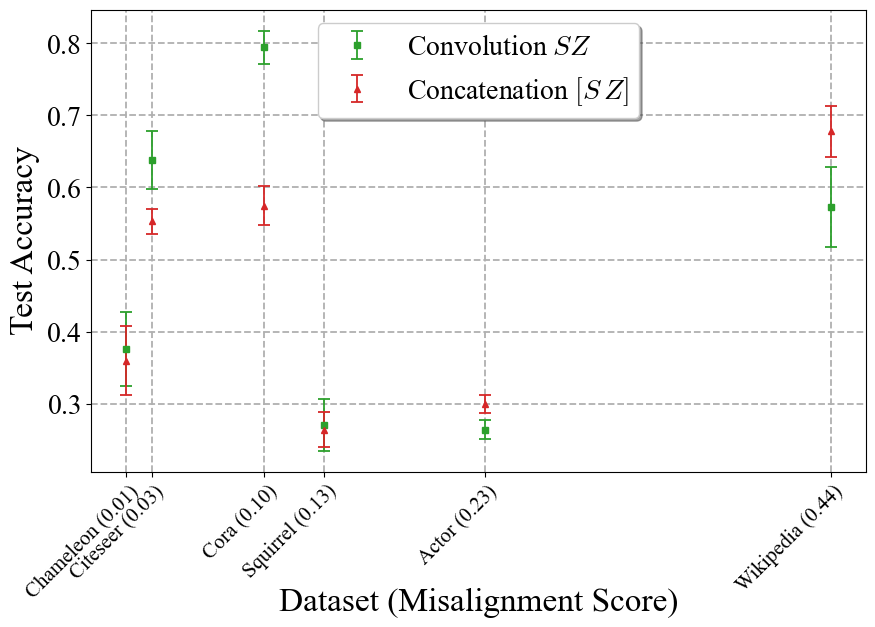}
    \caption{\textbf{Misalignment vs. Accuracy.} Test accuracy of convolution (\( {S}{Z} \)) and concatenation (\( [S \ Z] \)) models. The x-axis shows the misalignment score as defined in Definition \ref{def:Misalignment} for \( H = SZ \).}
    \label{fig:dataset_comparison}
    \vspace{-20pt} 
\end{wrapfigure}
Figure \ref{fig:dataset_comparison} shows the test accuracy of convolution (GCN denoted as \( SZ \)) and concatenation (\( [S \ Z] \)) across six real datasets---Cora \citep{10.5555/295240.295725}, Citeseer \citep{DBLP:conf/dl/GilesBL98}, Wikipedia (Wikipedia II from \cite{DBLP:journals/tnn/QianERPB22}), Squirrel \citep{DBLP:journals/compnet/RozemberczkiAS21}, Chameleon \citep{DBLP:journals/compnet/RozemberczkiAS21} and Actor \citep{DBLP:conf/kdd/TangSWY09}--- plotted against the misalignment score. 
We compute a normalised  misalignment score
\(\displaystyle
\frac{\operatorname{Tr}\left( (I - P_H) XX^\top \right)}{\operatorname{Tr}(XX^\top)}
\), assuming \(X\) is the one-hot encoded label matrix. Figure \ref{fig:dataset_comparison} shows that convolution performs poorly on a dataset with high misalignment, such as Wikipedia, whereas the concatenation can use the graph information better. In particular, Wikipedia is an example where node features carry most of the signal, and concatenation can utilize this even under strong misalignment. Actor has intermediate misalignment, where concatenation offers only a small benefit. On the other hand, benchmark datasets like Cora and Citeseer have low misalignment, and convolution performs well, which explains their reported success in GCN evaluations. These results shows that we should have a critical look at the benchmark datasets when we are evaluating the performance of GNNs. Most of these datasets have low misalignment and favor graph convolution, which limits our ability to assess their weaknesses. The misalignment score, stemming from the generalisation error, allows us to put these datasets under the microscope.  
Interestingly, the misalignment score is not necessarily high for heterophilic graphs such as Squirrel and Chameleon, which suggests that misalignment cannot fully capture the challenges posed by heterophilic structures. The following section investigates the complete error term $R_H$ to better understand how heterophily affects the model performance.

\subsection{Heterophilic Data: Why Some GNNs Fail}
\label{Ch:hetero}
Our framework introduced in Section~\ref{Ch:pre} allows us to model homophilic and heterophilic graphs by choosing the spectral function on the underlying data matrix $(\Lambda^{\star})$. In homophilic graphs, the labels change smoothly, creating a low-pass signal \citep{DBLP:journals/pieee/OrtegaFKMV18}, concentrated on the low-frequency eigenvectors of the Laplacian. Whereas, in heterophilic graphs, the labels change roughly across the edges, forming a high-frequency signal \citep{DBLP:conf/acml/ZhangL23}. That is, a homophilic graph has a low-pass spectrum while a heterophilic graph has a high-pass spectrum w.r.t the Laplacian. 

To formalize this distinction, we introduce the homophily parameter \( q \in [0, 1] \) 
where \( q = 0 \) corresponds to a perfectly heterophilic graph and \( q = 1 \) to a perfectly homophilic graph. Specifically, we consider the spectral function on the underlying data matrix $X$ and on the features $Z$: 
\begin{equation}
   \lambda^\star_{i} = (\lambda_f)_{i} = q - \frac{(2q - 1) \lambda_{i}}{2},  
    \label{eq:f_lambda}
\end{equation}
where \( \lambda_i \in [0,2] \), as these are the eigenvalues of the symmetrically normalized Laplacian. Consistently with the construction in Eq. \eqref{eq:XZ}, the \(n-d\) eigenvalues mapped to zero are those for which \(\lambda_i = 1\) i.e., neither a low- nor a high-frequency.
In the following, we investigate how the graph filter influences generalisation across varying homophily.
\begin{wrapfigure}[29]{r}{0.5\columnwidth}
    \centering
    \begin{subfigure}[t]{\linewidth}
        \centering
        \includegraphics[width=\linewidth]{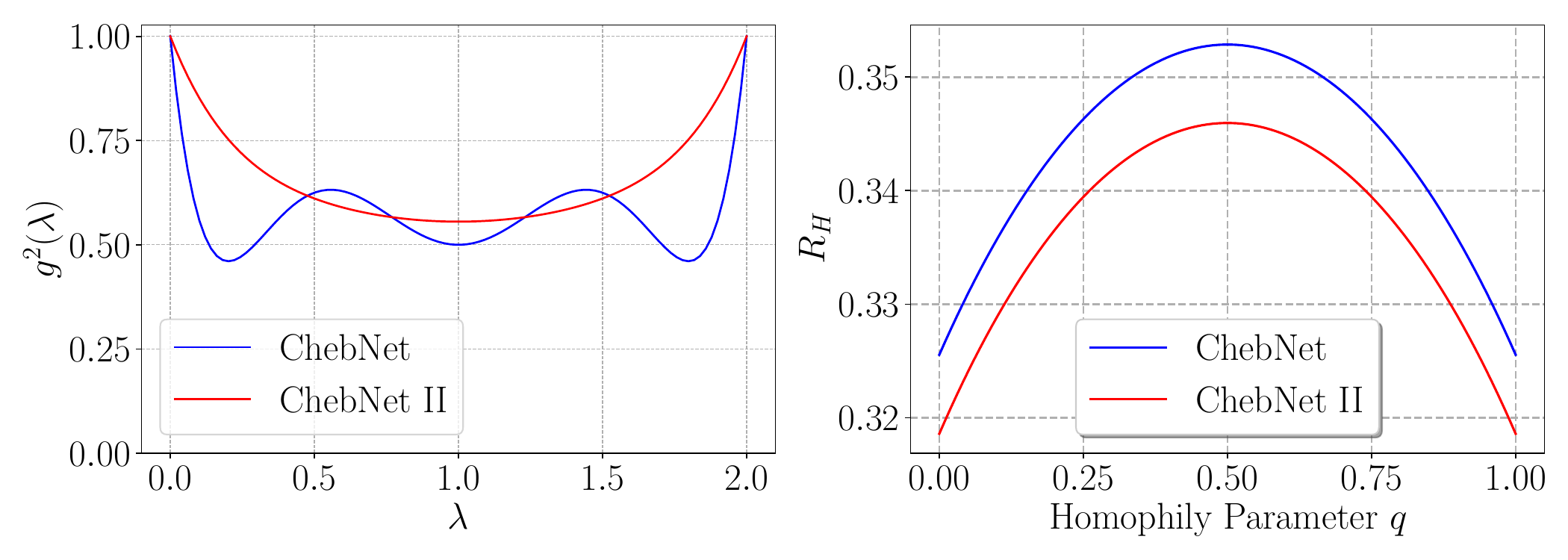}
        \caption{\textbf{Two variants of Chebyshev networks.} 
        \textbf{Left:} Frequency response showing the improvement of ChebNet II over ChebNet. 
        \textbf{Right:} $R_H$, averaged over number of eigenvalues, for ChebNet and ChebNet II over homophily parameter $q$.}
        \label{fig:chebnet_comparison}
    \end{subfigure}
    
    \vspace{0.5cm}

    \begin{subfigure}[t]{\linewidth}
        \centering
        \includegraphics[width=\linewidth]{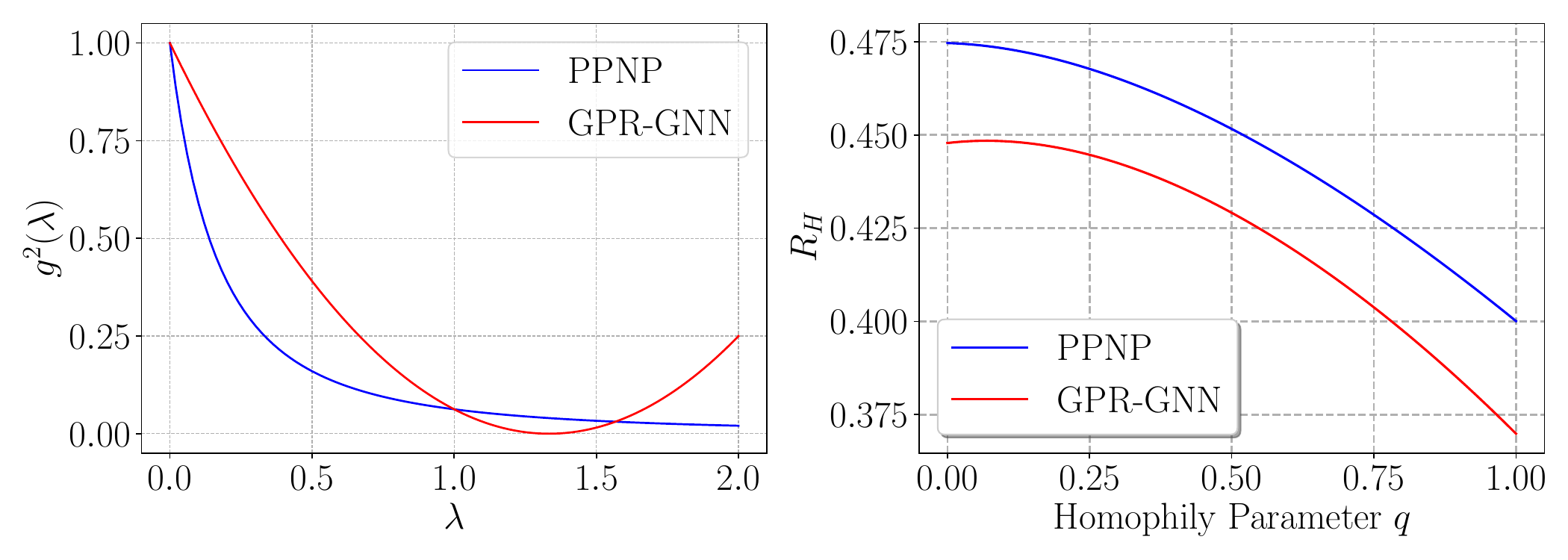}
        \caption{\textbf{Two variants of PageRank-based networks.} 
        \textbf{Left:} Frequency response showing the improvement of GPR-GNN over PPNP. 
        \textbf{Right:} $R_H$, averaged over number of eigenvalues, for GPR-GNN and PPNP over homophily parameter $q$.}
        \label{fig:pagerank_comparison}
    \end{subfigure}
    
    \caption{Frequency response and theoretical error of Chebyshev and PageRank-based GNNs.}
    \label{fig:freq_response}
\end{wrapfigure}

\textbf{ChebNet suffers from the Runge phenomenon.} 
\cite{DBLP:conf/nips/HeWW22} highlights that ChebNet II outperforms ChebNet by addressing the Runge phenomenon \citep{Epperson}, which causes high-amplitude oscillations near the boundaries of the approximation interval when using higher-degree Chebyshev polynomials. Our analysis of frequency responses (Figure \ref{fig:chebnet_comparison} left) shows that the oscillations are smoothed out for ChebNet II, therefore leading to a better theoretical error $R_H$ according to Corollary \ref{gen-err} (Figure \ref{fig:chebnet_comparison} right). 


\textbf{PPNP is worse at retaining high frequency than GPR-GNN.} 
Between the pagerank-based GNNs, PPNP and GPR-GNN, \cite{DBLP:conf/iclr/ChienP0M21} observes that PPNP usually performs worse than GPR-GNN and 
suggests that suppressing high-frequency components makes PPNP inadequate for heterophilic graphs. Using their frequency responses, we show that GPR-GNN retains some high-frequency components while PPNP has none (Figure \ref{fig:pagerank_comparison} left), which leads to better theoretical error $R_H$ for GPR-GNN (Figure \ref{fig:pagerank_comparison} right).


These findings show that investigating GNNs as filters provides insights into their ability to handle different levels of homophily, depending on whether the filters pass low or high frequencies (referred to as low-pass or high-pass filters). 
Exemplary, Chebyshev-based GNNs pass both low and high frequencies (Figure~\ref{fig:chebnet_comparison}) and thus, can handle strongly homophilic or heterophilic graphs but struggle in the intermediate range. In contrast, PPNP retains only low frequencies while GPR-GNN retains some high-frequency information (Figure~\ref{fig:pagerank_comparison}).
Analysis on other GNNs is in Appendix \ref{Ap:homo_exp}.
Appendix \ref{Ap:Oversmoothing} presents a filtering-based perspective on oversmoothing with insights into the limitations of different GNNs in relation to the spectral properties of the graph.

\textbf{GCN, a low pass filter, fails (mostly) on heterophilic data.}
GCNs act as low-pass filters, which intuitively makes them well-suited for homophilic graphs; however, a closer look reveals that this intuition does not always hold. As presented in Figure \ref{fig:intro}, GCN performs \emph{better} on a semi-synthetic more heterophilic Squirrel variant than on the original benchmark, and at the same time fails on perfectly homophilic duplicated-Cora with copied features, motivating a theoretical analysis of its generalisation behaviour. 
\begin{corollary}[\textbf{Derivative of Generalisation Error with Respect to Homophily}]
\label{Cor:derivative}
The derivative of the error in Corollary \ref{gen-err} with respect to \( q \) can be computed analytically for a single-layer GCN with normalized filter \( g(\Lambda) \) lying in \([0, 1]\), under the framework given in Section~\ref{Ch:pre} and \( h^{\star}(\Lambda) \) (and \( f(\Lambda) \)) in Equation~\ref{eq:f_lambda}. \(\displaystyle \frac{dR_{\text{GCN}}}{dq} \) is in Appendix \ref{Ap:cor42}.
\end{corollary}

\begin{wrapfigure}[11]{r}{0.5\columnwidth}
    \centering
    \vspace{-10pt} 
    \includegraphics[width=\linewidth]{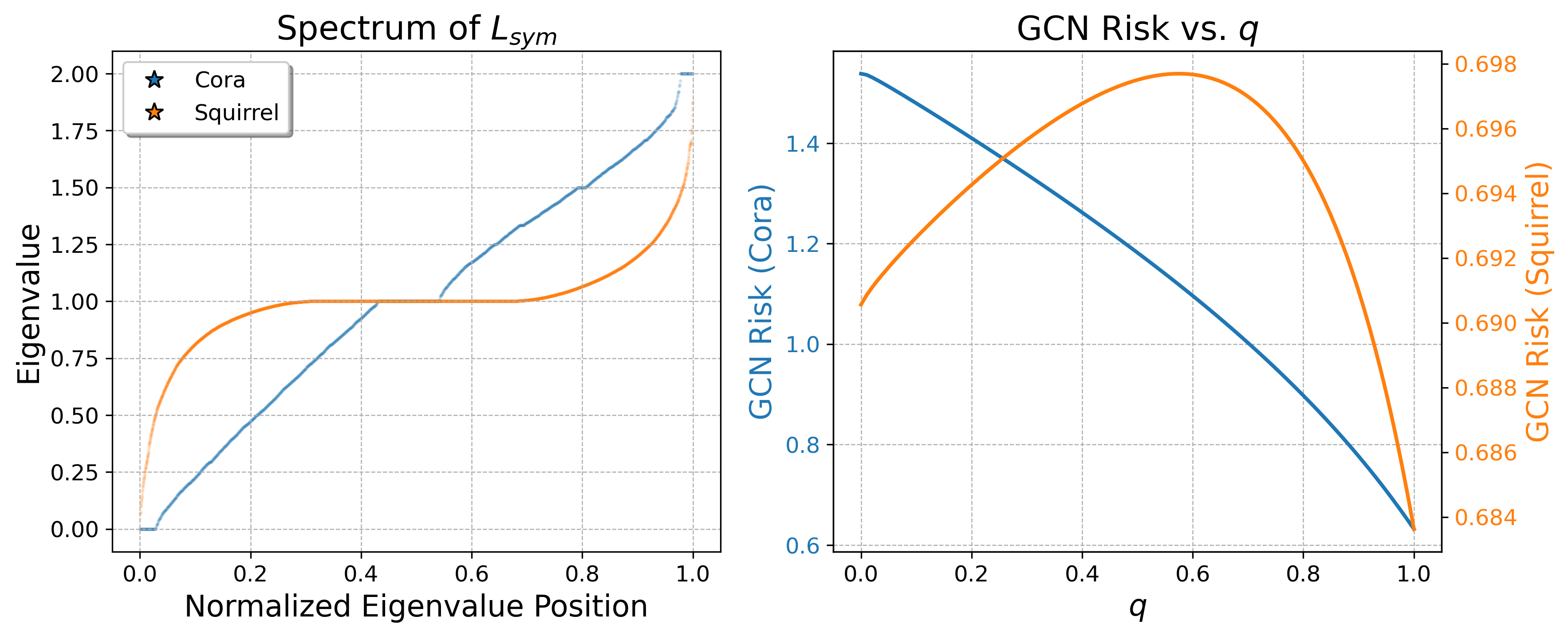}
    \caption{\textbf{Spectral Symmetry and Theoretical Error.} 
    Left: \(\lambda\) of Cora and Squirrel, showing approximate spectral symmetry. 
    Right: \(R_{\text{GCN}}\), averaged over the number of \(\lambda\), as a function of \(q\) with \(c=0.01\).}
    \label{fig:eigenvalues_and_gcn_risk_cora_squirrel}
    \vspace{1pt}
\end{wrapfigure}

Since the derivative of the generalisation error with respect to homophily has complex dependencies, we investigate its sign numerically. Refer to Appendix \ref{Ap:cor42} for the explicit derivative expression and solver queries.
\begin{remark}\textbf{Sign of the Derivative of Generalisation Error. }
 Consider graphs with a symmetric spectrum i.e., if \( \lambda \) is an eigenvalue, then so is \( \lambda_{\max} - \lambda \), where \( \lambda_{\max}\) is the largest eigenvalue of \(L\), and the eigenvalue \(1\) has multiplicity \(n-d\).
\textbf(i) if \( \lambda_{\max} = 2 \)  and \(c>0.1\) then
\(\displaystyle
    \frac{dR_{\text{GCN}}}{dq} < 0.
    \)
 \textbf(ii) 
if \textbf{\( \lambda_{\max} < 2\)}, then, for sufficiently small values of \( c \) (depending on \( 2-\lambda_{\max} \)),
    \(\displaystyle
    \frac{dR_{\text{GCN}}}{dq} > 0.
    \)
\label{Cor:Derivative}
\end{remark} Remark~\ref{Cor:Derivative} assumes spectral symmetry of \(L\), which generally does not always hold. However, as illustrated in Figure~\ref{fig:eigenvalues_and_gcn_risk_cora_squirrel}, the spectra of real-world graphs such as Cora and Squirrel are approximately symmetric, and the eigenvalue \(1\) has a large multiplicity. For Cora, where \(\lambda_{\max}=2\), the generalisation error in Corollary \ref{gen-err} decreases with increasing homophily \(q\). For Squirrel, which also has an approximately symmetric spectrum but \(\lambda_{\max} < 2\), the error neither decreases nor increases monotonically, since the term \(c\) is not sufficiently small. Additional experiments, including an ablation on the noise level are provided in Appendix \ref{Ap:homo_exp}, along with results on other GNN architectures.
Consequently, our spectral definition effectively captures the distinction between homophilic and heterophilic settings, even when the assumptions are only approximately satisfied.

\begin{wrapfigure}{r}{0.51\columnwidth}
    \centering
    \vspace{-15pt} 
    \includegraphics[width=\linewidth]{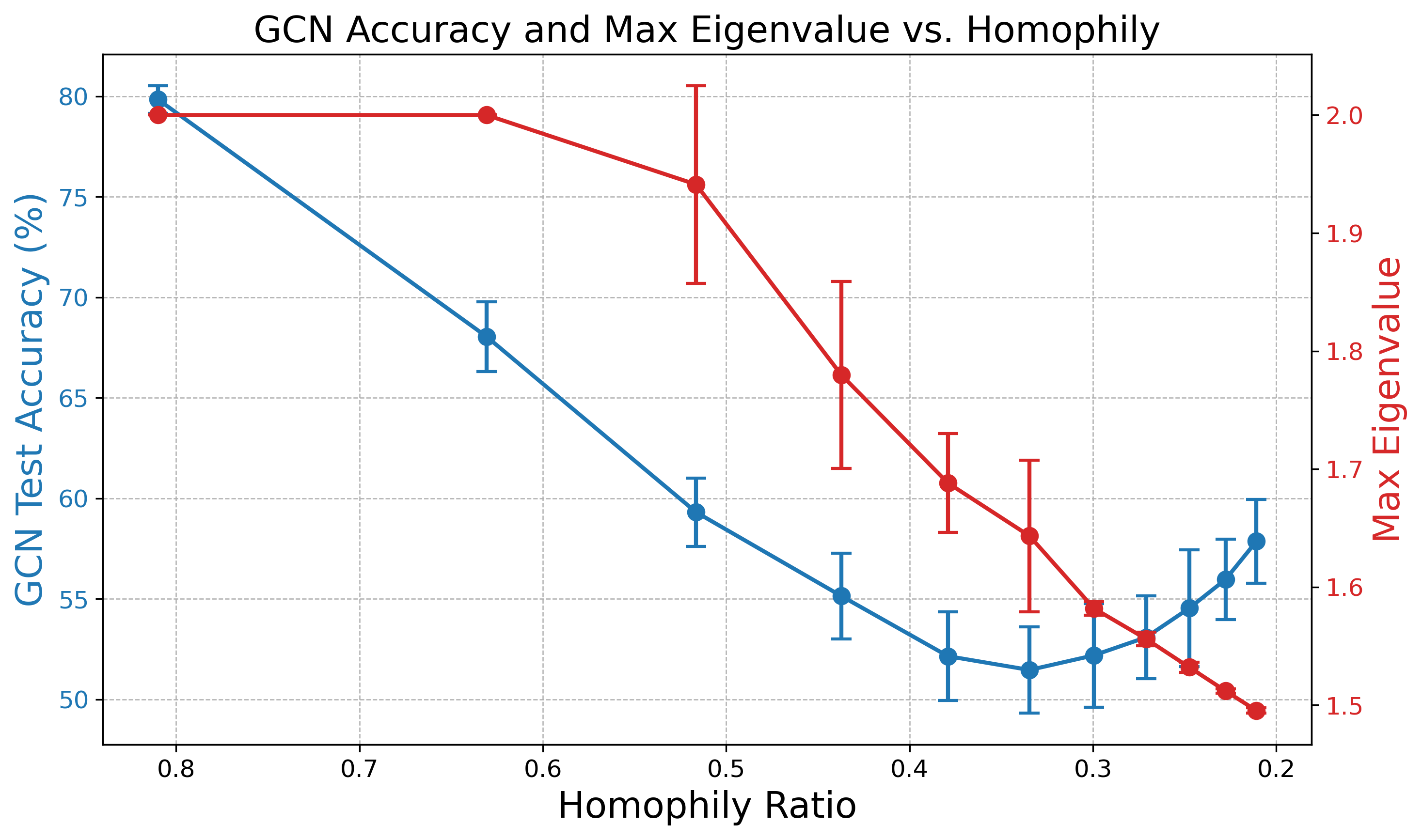}
    \caption{\textbf{Effect of Laplacian Spectrum on GCN Accuracy under Varying Homophily.} 
    GCN test accuracy (left) and the \(\lambda_{\max}\) (right).}
    \label{fig:acc_homophily}
    \vspace{-10pt} 
\end{wrapfigure}

To study how the Laplacian spectrum affects generalisation error with respect to homophily in a realistic setting, we conduct a perturbation experiment on Cora. We introduce heterophilic edges, sampled independently from a neighbourhood distribution that excludes neighbours with the same label. The homophily score is measured as the fraction of edges that connect nodes sharing the same label, as is standard for real-world graphs \citep{DBLP:conf/nips/ZhuYZHAK20}.
Figure~\ref{fig:acc_homophily} shows that the GCN accuracy initially decreases as homophily is reduced. However, the accuracy begins to recover once the largest eigenvalue becomes sufficiently small, which is consistent with the theoretical prediction in Remark \ref{Cor:Derivative}. A similar experiment indicating that GCNs perform well even under extreme heterophily appears in \cite{DBLP:conf/iclr/0001LST22} while our analysis provides mathematical backing to observe how homophily influences the performance of GCN and can be extended to other GNNs.

\subsection{Repeating Eigenvalues: What Does Graph Attention Miss}
\label{Ch:GAT}
In this section, we compare the attention-based models GAT \citep{DBLP:conf/iclr/VelickovicCCRLB18} and Specformer \citep{DBLP:conf/iclr/BoSWL23}, which both have a trainable convolution. We abstract this through the following trainable graph filter view of GAT and Specformer, which is in line with \cite{DBLP:conf/iclr/BalcilarRHGAH21,DBLP:journals/tmlr/BastosNSKSM22}, but mathematically more formal.

\begin{definition}[\textbf{Spectral View of Graph Attention Networks and Specformer}]
GAT and Specformer can be interpreted as a spectral filter with a trainable frequency response $g$.
GAT apply $g$ to individual eigenvalues $\lambda_i$ of the graph Laplacian, leading to the generalisation error defined as
\(\displaystyle
    R_{\text{GAT}} = \inf_{\substack{g: \lambda \mapsto \tilde{\lambda}}} R_{H}.
\) Specformer, in contrast, applies a set-to-set function $g$ to the entire spectrum $\Lambda$, transforming it to $\tilde{\Lambda} = g(\Lambda)$, leading to the generalisation error defined as
\( \displaystyle
    R_{\text{SF}} = \inf_{\substack{ g: \Lambda \mapsto \tilde{\Lambda}}} R_{H}.
\)
\textit{We consider an idealised version of GAT and Specformer, where the infimum is taken over \(R_H\).}
\label{Def:gat}
\end{definition}

Like traditional polynomial-based spectral GNNs \citep{DBLP:conf/aaai/LuY0LYGLYC24}, GAT aggregates information from immediate neighbours and faces a limitation when dealing with repeated eigenvalues \citep{DBLP:conf/iclr/BoSWL23}. GAT applies the same filtering to all eigenvectors associated with a repeated eigenvalue, potentially losing some structural information, which can be formalised using our framework as shown in the following corollary to Theorem \ref{thm:aniso}  (proof in Appendix \ref{Ap:gat_specformer}):
\begin{corollary}[\label{cor:gat_specformer}\textbf{Generalisation Error between GAT and Specformer}] Consider the generalisation error in Theorem~\ref{thm:aniso}. Let 
\(
C = \left\{
\begin{aligned}
&\mathcal{I} \subseteq [d] \,:\ 
|\mathcal{I}| \ge 2,\ 
\lambda_i = \lambda_j, \lambda_i^\star = \lambda_j^\star,\\
&\mathbb{E}\langle \theta^\star, \phi_i \rangle 
\neq \mathbb{E}\langle \theta^\star, \phi_j \rangle
\ \forall i,j\in\mathcal{I}
\end{aligned}
\right\}
\)
be the collection of index sets of repeated eigenvalues with different alignment of \(\theta^\star\). Then, the difference in generalisation errors between GAT and Specformer is given by  \(R_{\text{GAT}} - R_{\text{SF}} =\)
\(
\sum_{\substack{\mathcal{I} \in C}}\left[
\frac{ \left( \sum_{i \in \mathcal{I}} \lambda_i^\star \mathbb{E}\langle \theta^\star, \phi_i \rangle^2 \right) c}
{ \left( \frac{1}{|\mathcal{I}|} \sum_{i \in \mathcal{I}} \lambda_i^\star \mathbb{E}\langle \theta^\star, \phi_i \rangle^2 \right) + c}
- \sum_{i \in \mathcal{I}} \frac{ \lambda_i^\star \mathbb{E}\langle \theta^\star, \phi_i \rangle^2  c}
{ \lambda_i^\star \mathbb{E}\langle \theta^\star, \phi_i \rangle^2 + c}
\right].
\)
\end{corollary}

Corollary \ref{cor:gat_specformer} states that if there are at least two eigenvalues with the same value but different \(\mathbb{E}\langle \theta^\star, \phi_i \rangle^2\), the difference in generalisation error between GAT and Specformer is strictly positive (\(\displaystyle  R_{\text{SF}}<  R_{\text{GAT}}\)). GAT's inability to assign different frequency responses to eigenvectors corresponding to the same eigenvalue introduces an additional error.

\begin{wrapfigure}[18]{r}{0.5\columnwidth}
    \centering
    \vspace{-10pt} 
    \includegraphics[width=\linewidth]{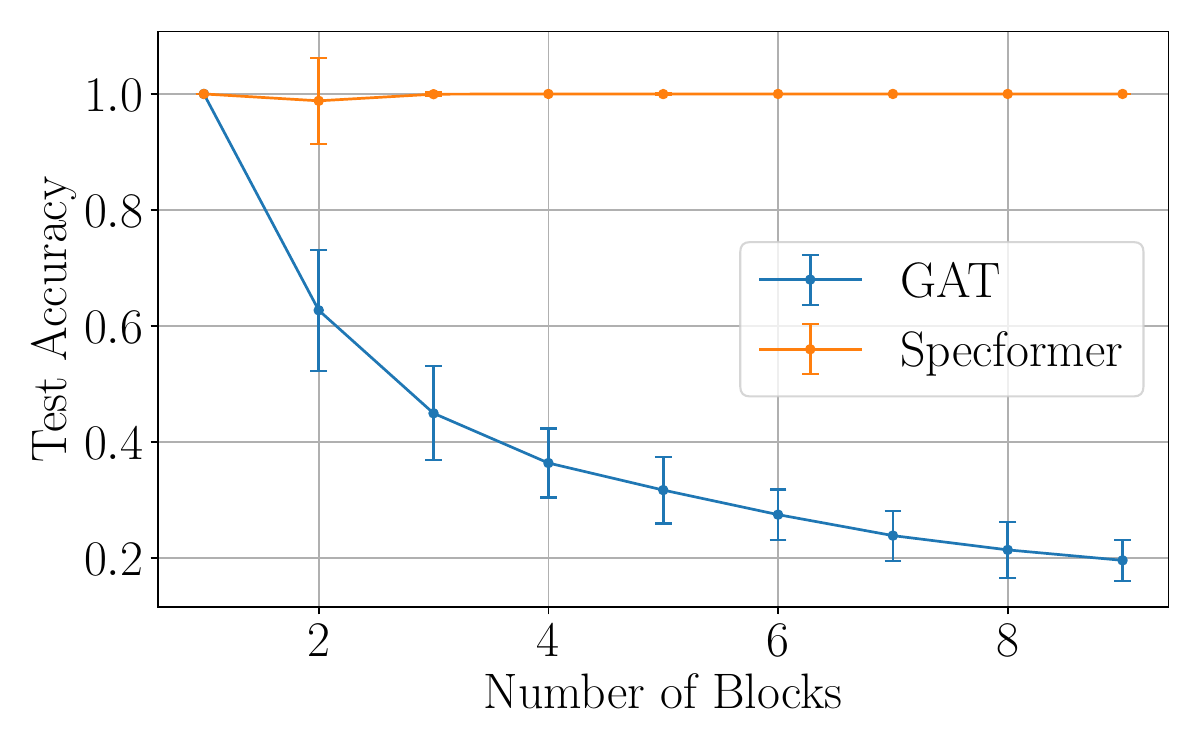}
    \caption{\textbf{Effect of Eigenvalue Multiplicity on GAT and Specformer Performance.} 
    Accuracy (averaged over 50 runs) as a function of the number of cycle graph blocks is plotted. 
    As more blocks are added, the multiplicity of eigenvalues in the Laplacian spectrum increases.}
    \label{fig:repeating_eigenvalue}
\end{wrapfigure}
To empirically validate Corollary~\ref{cor:gat_specformer}, we compare Specformer and GAT on a synthetic dataset constructed from disjoint cycle graphs with random node features, where all nodes in a cycle share the same label. At each step, we grow the graph by adding a new cycle graph and a unique label, thus increasing the total number of labels with every addition, as detailed in Appendix \ref{Ap:Experiments}. 
One can consider the one-hot encoded label matrix as the underlying data matrix, inducing a block-diagonal structure with one rank-one block per cycle; each added cycle increases the multiplicity of the zero eigenvalue and the multiplicity of Laplacian eigenvalues. As illustrated in Figure~\ref{fig:repeating_eigenvalue}, Specformer has near-perfect accuracy, while the performance of GAT decreases with the number of repeated eigenvalues. 
This experiment clearly supports our theoretical finding that Specformer generalises better than GAT in the presence of repeating eigenvalues.

\section{Conclusion}
\label{Ch:conclusion}
By adopting a signal processing perspective on GNNs, we derive the \emph{exact} generalisation error and gain valuable insights into the conditions under which GNNs, including attention-based models, can effectively leverage graph data.
\begin{mybox}{hellblue}{\textcolor{black}{Key Theoretical Insights}}
$1.$ Generalisation depends strongly on the \textbf{alignment between graph structure and node features}. \\
$2.$ Homophily affects generalisation differently across architectures. \\
$3.$ GCN performance can improve under \textbf{extreme heterophily}, contrary to common intuition.
\end{mybox}
Our insights from the theory highlight an important inherent bias in the current benchmarking. Most theoretical works on understanding GNNs \citep{shi2024homophily,DBLP:conf/iclr/WuCWJ23,luan2023graph}, discussed in Appendix \ref{Ap:related_work}, assume the Contextual Stochastic Block Model (CSBM), which presupposes alignment between graph and features \citep{DBLP:conf/nips/DeshpandeSMM18}. This mirrors the nature of most benchmark datasets, leading to the development of GNNs that rely on alignment for better performance. Consequently, they perform poorly in cases where the alignment is weak, as demonstrated in Lemma \ref{lemma:misalignment}, which raises the key question:  
\emph{Are graph convolutions truly optimal for combining feature and structure information across diverse graph regimes?}

A comparable trend was observed earlier, with many benchmark datasets being predominantly homophilic \citep{DBLP:journals/corr/abs-2104-01404}, sparking the debate on the progress of GNNs on heterophilic graphs \citep{DBLP:journals/corr/abs-1905-09550}. 
These observations underscore the need for caution in benchmarking, with dataset characteristics, such as misalignment and homophily, serving as tools to evaluate datasets more critically. However, a rigorous theoretical understanding is still lacking, which our analysis begins to address. The gaps between theory and practice persist, as many empirical observations about GNNs such as oversquashing \citep{DBLP:conf/iclr/0002Y21} and over-globalisation \cite{DBLP:conf/icml/Xing0LHS24} have not been fully formally explained. Our results emphasise that bridging these gaps is essential to prevent experiment-driven intuition from producing misleading insights.

\section*{Broader Impact}
This study provides the first exact generalisation error for GNNs, exposing strong benchmark biases that can skew the success or failure of different architectures. These results have no foreseeable negative societal impact. On the contrary, they offer theoretical tools for principled model selection and critical evaluation of existing benchmarks, advancing the reliable use of GNNs across scientific and engineering applications.

\section*{Reproducibility}
All code and configurations used for the experiments are available at the following link:
\url{https://figshare.com/s/61f8fafb9469750f173e}
\bibliography{references}

@article{khemani2024review,
  title={A review of graph neural networks: concepts, architectures, techniques, challenges, datasets, applications, and future directions},
  author={Khemani, Bharti and Patil, Shruti and Kotecha, Ketan and Tanwar, Sudeep},
  journal={Journal of Big Data},
  year={2024},
  publisher={Springer}

}

@misc{chen2022demystifyinggraphconvolutionsimple,
      title={Demystifying Graph Convolution with a Simple Concatenation}, 
      author={Zhiqian Chen and Zonghan Zhang},
      year={2022},
      eprint={2207.12931},
      archivePrefix={arXiv},
      primaryClass={cs.LG}
}

@article{DBLP:journals/corr/abs-1905-12921,
  author       = {Yifan Qian and
                  Paul Expert and
                  Tom Rieu and
                  Pietro Panzarasa and
                  Mauricio Barahona},
  title        = {Quantifying the alignment of graph and features in deep learning},
  journal      = {CoRR},
  year         = {2019}
}

@inproceedings{DBLP:conf/iclr/BalcilarRHGAH21,
  author       = {Muhammet Balcilar and
                  Guillaume Renton and
                  Pierre H{\'{e}}roux and
                  Benoit Ga{\"{u}}z{\`{e}}re and
                  S{\'{e}}bastien Adam and
                  Paul Honeine},
  title        = {Analyzing the Expressive Power of Graph Neural Networks in a Spectral
                  Perspective},
  booktitle    = {9th International Conference on Learning Representations},
  year         = {2021}

}

@inproceedings{DBLP:conf/iclr/KipfW17,
  author       = {Thomas N. Kipf and
                  Max Welling},
  title        = {Semi-Supervised Classification with Graph Convolutional Networks},
  booktitle    = {5th International Conference on Learning Representations},
  year         = {2017}
}

@inproceedings{DBLP:conf/iclr/KlicperaBG19,
  author       = {Johannes Klicpera and
                  Aleksandar Bojchevski and
                  Stephan G{\"{u}}nnemann},
  title        = {Predict then Propagate: Graph Neural Networks meet Personalized PageRank},
  booktitle    = {7th International Conference on Learning Representations},
  year         = {2019}
}

@article{DBLP:journals/pieee/OrtegaFKMV18,
  author       = {Antonio Ortega and
                  Pascal Frossard and
                  Jelena Kovacevic and
                  Jos{\'{e}} M. F. Moura and
                  Pierre Vandergheynst},
  title        = {Graph Signal Processing: Overview, Challenges, and Applications},
  journal      = {Proc. {IEEE}},
  year         = {2018}
}

@inproceedings{DBLP:conf/acml/ZhangL23,
  author       = {Acong Zhang and
                  Ping Li},
  title        = {Unleashing the Power of High-pass Filtering in Continuous Graph Neural
                  Networks},
  booktitle    = {Asian Conference on Machine Learning},
  series       = {Proceedings of Machine Learning Research},
  publisher    = {{PMLR}},
  year         = {2023}
}

@inproceedings{DBLP:conf/sspr/AnsarizadehTTR20,
  author       = {Fatemeh Ansarizadeh and
                  David B. H. Tay and
                  Dhananjay R. Thiruvady and
                  Antonio Robles{-}Kelly},
  title        = {Augmenting Graph Convolutional Neural Networks with Highpass Filters},
  booktitle    = {Structural, Syntactic, and Statistical Pattern Recognition - Joint
                  {IAPR} International Workshops},
  year         = {2020}
}

@inproceedings{DBLP:conf/iclr/VelickovicCCRLB18,
  author       = {Petar Velickovic and
                  Guillem Cucurull and
                  Arantxa Casanova and
                  Adriana Romero and
                  Pietro Li{\`{o}} and
                  Yoshua Bengio},
  title        = {Graph Attention Networks},
  booktitle    = {6th International Conference on Learning Representations},
  year         = {2018}
}

@inproceedings{DBLP:conf/iclr/BoSWL23,
  author       = {Deyu Bo and
                  Chuan Shi and
                  Lele Wang and
                  Renjie Liao},
  title        = {Specformer: Spectral Graph Neural Networks Meet Transformers},
  booktitle    = {The Eleventh International Conference on Learning Representations},
  year         = {2023}
}

@inproceedings{DBLP:conf/nips/HamiltonYL17,
  author       = {William L. Hamilton and
                  Zhitao Ying and
                  Jure Leskovec},
  title        = {Inductive Representation Learning on Large Graphs},
  booktitle    = {Advances in Neural Information Processing Systems 30: Annual Conference
                  on Neural Information Processing Systems},
  year         = {2017}
}

@article{DBLP:journals/tsp/LevieMBB19,
  author       = {Ron Levie and
                  Federico Monti and
                  Xavier Bresson and
                  Michael M. Bronstein},
  title        = {CayleyNets: Graph Convolutional Neural Networks With Complex Rational
                  Spectral Filters},
  journal      = {{IEEE} Trans. Signal Process.},
  year         = {2019}
}

@inproceedings{DBLP:conf/nips/DefferrardBV16,
  author       = {Micha{\"{e}}l Defferrard and
                  Xavier Bresson and
                  Pierre Vandergheynst},
  title        = {Convolutional Neural Networks on Graphs with Fast Localized Spectral
                  Filtering},
  booktitle    = {Advances in Neural Information Processing Systems 29: Annual Conference
                  on Neural Information Processing Systems},
  year         = {2016}
}

@inproceedings{DBLP:conf/aaai/BoWSS21,
  author       = {Deyu Bo and
                  Xiao Wang and
                  Chuan Shi and
                  Huawei Shen},
  title        = {Beyond Low-frequency Information in Graph Convolutional Networks},
  booktitle    = {Thirty-Fifth {AAAI} Conference on Artificial Intelligence},
  year         = {2021}
}

@inproceedings{DBLP:conf/iclr/ChienP0M21,
  author       = {Eli Chien and
                  Jianhao Peng and
                  Pan Li and
                  Olgica Milenkovic},
  title        = {Adaptive Universal Generalized PageRank Graph Neural Network},
  booktitle    = {9th International Conference on Learning Representations},
  year         = {2021}
}

@inproceedings{DBLP:conf/iclr/XuHLJ19,
  author       = {Keyulu Xu and
                  Weihua Hu and
                  Jure Leskovec and
                  Stefanie Jegelka},
  title        = {How Powerful are Graph Neural Networks?},
  booktitle    = {7th International Conference on Learning Representations},
  year         = {2019}
}

@inproceedings{DBLP:conf/aaai/0001RFHLRG19,
  author       = {Christopher Morris and
                  Martin Ritzert and
                  Matthias Fey and
                  William L. Hamilton and
                  Jan Eric Lenssen and
                  Gaurav Rattan and
                  Martin Grohe},
  title        = {Weisfeiler and Leman Go Neural: Higher-Order Graph Neural Networks},
  booktitle    = {The Thirty-Third {AAAI} Conference on Artificial Intelligence},
  year         = {2019}
}

@inproceedings{DBLP:conf/icml/WuSZFYW19,
  author       = {Felix Wu and
                  Amauri H. Souza Jr. and
                  Tianyi Zhang and
                  Christopher Fifty and
                  Tao Yu and
                  Kilian Q. Weinberger},
  title        = {Simplifying Graph Convolutional Networks},
  booktitle    = {Proceedings of the 36th International Conference on Machine Learning},
  series       = {Proceedings of Machine Learning Research},
  year         = {2019}
}

@inproceedings{DBLP:conf/kdd/VermaZ19,
  author       = {Saurabh Verma and
                  Zhi{-}Li Zhang},
  title        = {Stability and Generalization of Graph Convolutional Neural Networks},
  booktitle    = {Proceedings of the 25th {ACM} {SIGKDD} International Conference on
                  Knowledge Discovery {\&} Data Mining},
  year         = {2019}
}

@article{DBLP:journals/ijon/ZhouW21,
  author       = {Xianchen Zhou and
                  Hongxia Wang},
  title        = {The generalization error of graph convolutional networks may enlarge
                  with more layers},
  journal      = {Neurocomputing},
  year         = {2021}
}

@article{DBLP:journals/corr/abs-2102-10234,
  author       = {Shaogao Lv},
  title        = {Generalization bounds for graph convolutional neural networks via
                  Rademacher complexity},
  journal      = {CoRR},
  year         = {2021}
}

@inproceedings{DBLP:conf/iclr/LiaoUZ21,
  author       = {Renjie Liao and
                  Raquel Urtasun and
                  Richard S. Zemel},
  title        = {A PAC-Bayesian Approach to Generalization Bounds for Graph Neural
                  Networks},
  booktitle    = {9th International Conference on Learning Representations},
  year         = {2021}
}

@inproceedings{DBLP:conf/nips/EsserVG21,
  author       = {Pascal Mattia Esser and
                  Leena C. Vankadara and
                  Debarghya Ghoshdastidar},
  title        = {Learning Theory Can (Sometimes) Explain Generalisation in Graph Neural
                  Networks},
  booktitle    = {Advances in Neural Information Processing Systems 34: Annual Conference
                  on Neural Information Processing Systems},
  year         = {2021}
}

@article{DBLP:journals/corr/abs-1906-11300,
  author       = {Peter L. Bartlett and
                  Philip M. Long and
                  G{\'{a}}bor Lugosi and
                  Alexander Tsigler},
  title        = {Benign Overfitting in Linear Regression},
  journal      = {CoRR},
  year         = {2019}
}

@article{DBLP:journals/corr/abs-1905-09550,
  author       = {Hoang NT and
                  Takanori Maehara},
  title        = {Revisiting Graph Neural Networks: All We Have is Low-Pass Filters},
  journal      = {CoRR},
  year         = {2019}
}

@inproceedings{NIPS2001_d68a1827,
 author = {Sollich, Peter},
 booktitle = {Advances in Neural Information Processing Systems},
 publisher = {MIT Press},
 title = {Gaussian Process Regression with Mismatched Models},
 year = {2001}
}

@article{DBLP:journals/tmlr/BastosNSKSM22,
  author       = {Anson Bastos and
                  Abhishek Nadgeri and
                  Kuldeep Singh and
                  Hiroki Kanezashi and
                  Toyotaro Suzumura and
                  Isaiah Onando Mulang'},
  title        = {How Expressive are Transformers in Spectral Domain for Graphs?},
  journal      = {Trans. Mach. Learn. Res.},
  year         = {2022}
}

@inproceedings{DBLP:conf/aaai/LuY0LYGLYC24,
  author       = {Kangkang Lu and
                  Yanhua Yu and
                  Hao Fei and
                  Xuan Li and
                  Zixuan Yang and
                  Zirui Guo and
                  Meiyu Liang and
                  Mengran Yin and
                  Tat{-}Seng Chua},
  title        = {Improving Expressive Power of Spectral Graph Neural Networks with
                  Eigenvalue Correction},
  booktitle    = {Thirty-Eighth {AAAI} Conference on Artificial Intelligence},
  year         = {2024}
}

@article{buterez2024transfer,
  title={Transfer learning with graph neural networks for improved molecular property prediction in the multi-fidelity setting},
  author={Buterez, David and Janet, Jon Paul and Kiddle, Steven J and Oglic, Dino and Li{\'o}, Pietro},
  journal={Nature communications},
  year={2024}
}

@inproceedings{DBLP:conf/www/Fan0LHZTY19,
  author       = {Wenqi Fan and
                  Yao Ma and
                  Qing Li and
                  Yuan He and
                  Yihong Eric Zhao and
                  Jiliang Tang and
                  Dawei Yin},
  title        = {Graph Neural Networks for Social Recommendation},
  booktitle    = {The World Wide Web Conference},
  year         = {2019}
}

@inproceedings{DBLP:conf/nips/Keriven22,
  author       = {Nicolas Keriven},
  title        = {Not too little, not too much: a theoretical analysis of graph (over)smoothing},
  booktitle    = {Advances in Neural Information Processing Systems 35: Annual Conference
                  on Neural Information Processing Systems },
  year         = {2022}
}

@book{bach2024learning,
  title={Learning theory from first principles},
  author={Bach, Francis},
  year={2024},
  publisher={MIT press}
}

@article{DBLP:journals/focm/CaponnettoV07,
  author       = {Andrea Caponnetto and
                  Ernesto De Vito},
  title        = {Optimal Rates for the Regularized Least-Squares Algorithm},
  journal      = {Found. Comput. Math.},
  year         = {2007}
}

@inproceedings{DBLP:conf/iclr/0001LST22,
  author       = {Yao Ma and
                  Xiaorui Liu and
                  Neil Shah and
                  Jiliang Tang},
  title        = {Is Homophily a Necessity for Graph Neural Networks?},
  booktitle    = {The Tenth International Conference on Learning Representations},
  year         = {2022}
}

@inproceedings{DBLP:conf/nips/ZhuYZHAK20,
  author       = {Jiong Zhu and
                  Yujun Yan and
                  Lingxiao Zhao and
                  Mark Heimann and
                  Leman Akoglu and
                  Danai Koutra},
  title        = {Beyond Homophily in Graph Neural Networks: Current Limitations and
                  Effective Designs},
  booktitle    = {Advances in Neural Information Processing Systems 33: Annual Conference
                  on Neural Information Processing Systems},
  year         = {2020}
}

@article{Epperson,
author = {James F. Epperson},
title = {On the Runge Example},
journal = {The American Mathematical Monthly},
year = {1987},
publisher = {Taylor \& Francis}
}

@article{DBLP:journals/compnet/RozemberczkiAS21,
  author       = {Benedek Rozemberczki and
                  Carl Allen and
                  Rik Sarkar},
  title        = {Multi-Scale attributed node embedding},
  journal      = {J. Complex Networks},
  year         = {2021}
}

@article{shi2024homophily,
  title={Homophily modulates double descent generalization in graph convolution networks},
  author={Shi, Cheng and Pan, Liming and Hu, Hong and Dokmani{\'c}, Ivan},
  journal={Proceedings of the National Academy of Sciences},
  year={2024},
  publisher={National Academy of Sciences}
}

@inproceedings{DBLP:conf/iclr/WuCWJ23,
  author       = {Xinyi Wu and
                  Zhengdao Chen and
                  William Wei Wang and
                  Ali Jadbabaie},
  title        = {A Non-Asymptotic Analysis of Oversmoothing in Graph Neural Networks},
  booktitle    = {The Eleventh International Conference on Learning Representations},
  year         = {2023}
}

@article{luan2023graph,
  title={When do graph neural networks help with node classification? investigating the homophily principle on node distinguishability},
  author={Luan, Sitao and Hua, Chenqing and Xu, Minkai and Lu, Qincheng and Zhu, Jiaqi and Chang, Xiao-Wen and Fu, Jie and Leskovec, Jure and Precup, Doina},
  journal={Advances in Neural Information Processing Systems},
  year={2023}
}

@inproceedings{DBLP:conf/nips/DeshpandeSMM18,
  author       = {Yash Deshpande and
                  Subhabrata Sen and
                  Andrea Montanari and
                  Elchanan Mossel},
  title        = {Contextual Stochastic Block Models},
  booktitle    = {Advances in Neural Information Processing Systems 31: Annual Conference
                  on Neural Information Processing Systems },
  year         = {2018}
}

@inproceedings{DBLP:conf/iclr/RauchwergerJL25,
  author       = {Levi Rauchwerger and
                  Stefanie Jegelka and
                  Ron Levie},
  title        = {Generalization, Expressivity, and Universality of Graph Neural Networks
                  on Attributed Graphs},
  booktitle    = {The Thirteenth International Conference on Learning Representations},
  year         = {2025}
}

@inproceedings{DBLP:conf/iclr/LiG0025,
  author       = {Shouheng Li and
                  Floris Geerts and
                  Dongwoo Kim and
                  Qing Wang},
  title        = {Towards Bridging Generalization and Expressivity of Graph Neural Networks},
  booktitle    = {The Thirteenth International Conference on Learning Representations},
  year         = {2025}
}

@inproceedings{DBLP:conf/dl/GilesBL98,
  author       = {C. Lee Giles and
                  Kurt D. Bollacker and
                  Steve Lawrence},
  title        = {CiteSeer: An Automatic Citation Indexing System},
  booktitle    = {Proceedings of the 3rd {ACM} International Conference on Digital Libraries},
  year         = {1998}
}

@inproceedings{10.5555/295240.295725,
author = {Craven, Mark and DiPasquo, Dan and Freitag, Dayne and McCallum, Andrew and Mitchell, Tom and Nigam, Kamal and Slattery, Se\'{a}n},
title = {Learning to extract symbolic knowledge from the World Wide Web},
year = {1998},
publisher = {American Association for Artificial Intelligence},
 booktitle    = {Proceedings of the Fifteenth National Conference on Artificial Intelligence
                  and Tenth Innovative Applications of Artificial Intelligence Conference},
}

@article{DBLP:journals/tmlr/SabanayagamEG23,
  author       = {Mahalakshmi Sabanayagam and
                  Pascal Mattia Esser and
                  Debarghya Ghoshdastidar},
  title        = {Analysis of Convolutions, Non-linearity and Depth in Graph Neural
                  Networks using Neural Tangent Kernel},
  journal      = {Transactions of Machine Learning Research},
  year         = {2023}
}

@INPROCEEDINGS{9892655,
  author={Zopf, Markus},
  booktitle={2022 International Joint Conference on Neural Networks (IJCNN)}, 
  title={1-WL Expressiveness Is (Almost) All You Need}, 
  year={2022}}

@inproceedings{DBLP:conf/icml/TangL23,
  author       = {Huayi Tang and
                  Yong Liu},
  editor       = {Andreas Krause and
                  Emma Brunskill and
                  Kyunghyun Cho and
                  Barbara Engelhardt and
                  Sivan Sabato and
                  Jonathan Scarlett},
  title        = {Towards Understanding Generalization of Graph Neural Networks},
  booktitle    = {International Conference on Machine Learning},
  publisher    = {{PMLR}},
  year         = {2023}
}

@book{schott2016matrix,
  title={Matrix analysis for statistics},
  author={Schott, James R},
  year={2016},
  publisher={John Wiley \& Sons}
}

@inproceedings{10.5555/3504035.3504468,
author = {Li, Qimai and Han, Zhichao and Wu, Xiao-Ming},
title = {Deeper insights into graph convolutional networks for semi-supervised learning},
year = {2018},
booktitle = {AAAI Conference on Artificial Intelligence}
}

@article{DBLP:journals/corr/abs-2303-10993,
  author       = {T. Konstantin Rusch and
                  Michael M. Bronstein and
                  Siddhartha Mishra},
  title        = {A Survey on Oversmoothing in Graph Neural Networks},
  journal      = {CoRR},
  volume       = {abs/2303.10993},
  year         = {2023}
}

@article{DBLP:journals/corr/abs-2008-09864,
  author       = {Wenbing Huang and
                  Yu Rong and
                  Tingyang Xu and
                  Fuchun Sun and
                  Junzhou Huang},
  title        = {Tackling Over-Smoothing for General Graph Convolutional Networks},
  journal      = {CoRR},
  volume       = {abs/2008.09864},
  year         = {2020}
}

@inproceedings{DBLP:conf/iclr/OonoS20,
  author       = {Kenta Oono and
                  Taiji Suzuki},
  title        = {Graph Neural Networks Exponentially Lose Expressive Power for Node
                  Classification},
  booktitle    = {8th International Conference on Learning Representations, {ICLR} 2020,
                  Addis Ababa, Ethiopia, April 26-30, 2020},
  publisher    = {OpenReview.net},
  year         = {2020}
}

@inproceedings{DBLP:conf/iclr/ToppingGC0B22,
  author       = {Jake Topping and
                  Francesco Di Giovanni and
                  Benjamin Paul Chamberlain and
                  Xiaowen Dong and
                  Michael M. Bronstein},
  title        = {Understanding over-squashing and bottlenecks on graphs via curvature},
  booktitle    = {The Tenth International Conference on Learning Representations, {ICLR}
                  2022, Virtual Event, April 25-29, 2022},
  publisher    = {OpenReview.net},
  year         = {2022}
}

@InProceedings{pmlr-v202-black23a,
  title = 	 {Understanding Oversquashing in {GNN}s through the Lens of Effective Resistance},
  author =       {Black, Mitchell and Wan, Zhengchao and Nayyeri, Amir and Wang, Yusu},
  booktitle = 	 {Proceedings of the 40th International Conference on Machine Learning},
  year = 	 {2023},
  volume = 	 {202},
  series = 	 {Proceedings of Machine Learning Research},
}

@article{koltchinskii2017concentration,
 author = {VLADIMIR KOLTCHINSKII and KARIM LOUNICI},
 journal = {Bernoulli},
 title = {Concentration inequalities and moment bounds for sample covariance operators},
 year = {2017}
}

@article{stewart1977perturbation,
  title={On the perturbation of pseudo-inverses, projections and linear least squares problems},
  author={Stewart, Gilbert W},
  journal={SIAM review},
  year={1977}
}

@article{DBLP:journals/corr/abs-2104-01404,
  author       = {Derek Lim and
                  Xiuyu Li and
                  Felix Hohne and
                  Ser{-}Nam Lim},
  title        = {New Benchmarks for Learning on Non-Homophilous Graphs},
  journal      = {CoRR},
  year         = {2021}
}

@article{DBLP:journals/corr/abs-2502-14546,
  author       = {Maya Bechler{-}Speicher and
                  Ben Finkelshtein and
                  Fabrizio Frasca and
                  Luis M{\"{u}}ller and
                  Jan T{\"{o}}nshoff and
                  Antoine Siraudin and
                  Viktor Zaverkin and
                  Michael M. Bronstein and
                  Mathias Niepert and
                  Bryan Perozzi and
                  Mikhail Galkin and
                  Christopher Morris},
  title        = {Position: Graph Learning Will Lose Relevance Due To Poor Benchmarks},
  journal      = {CoRR},
  year         = {2025}
}

@article{DBLP:journals/corr/abs-2502-02379,
  author       = {Corinna Coupette and
                  Jeremy Wayland and
                  Emily Simons and
                  Bastian Rieck},
  title        = {No Metric to Rule Them All: Toward Principled Evaluations of Graph-Learning
                  Datasets},
  journal      = {CoRR},
  volume       = {abs/2502.02379},
  year         = {2025}
}

@article{DBLP:journals/tnn/QianERPB22,
  author       = {Yifan Qian and
                  Paul Expert and
                  Tom Rieu and
                  Pietro Panzarasa and
                  Mauricio Barahona},
  title        = {Quantifying the Alignment of Graph and Features in Deep Learning},
  journal      = {{IEEE} Trans. Neural Networks Learn. Syst.},
  year         = {2022}
}

@misc{klepper2023relatinggraphautoencoderslinear,
      title={Relating graph auto-encoders to linear models}, 
      author={Solveig Klepper and Ulrike von Luxburg},
      year={2023},
      eprint={2211.01858},
      archivePrefix={arXiv},
      primaryClass={cs.LG}
}

@article{DBLP:journals/corr/abs-2406-08466,
  author       = {Licong Lin and
                  Jingfeng Wu and
                  Sham M. Kakade and
                  Peter L. Bartlett and
                  Jason D. Lee},
  title        = {Scaling Laws in Linear Regression: Compute, Parameters, and Data},
  journal      = {CoRR},
  year         = {2024}
}

@article{DBLP:journals/tnn/WuPCLZY21,
  author       = {Zonghan Wu and
                  Shirui Pan and
                  Fengwen Chen and
                  Guodong Long and
                  Chengqi Zhang and
                  Philip S. Yu},
  title        = {A Comprehensive Survey on Graph Neural Networks},
  journal      = {{IEEE} Trans. Neural Networks Learn. Syst.},
  year         = {2021}
}

@inproceedings{DBLP:conf/nips/HeWW22,
  author       = {Mingguo He and
                  Zhewei Wei and
                  Ji{-}Rong Wen},
  title        = {Convolutional Neural Networks on Graphs with Chebyshev Approximation,
                  Revisited},
  booktitle    = {Advances in Neural Information Processing Systems 35: Annual Conference
                  on Neural Information Processing Systems},
  year         = {2022}
}

@article{wu2022graph,
  title={Graph neural networks in recommender systems: a survey},
  author={Wu, Shiwen and Sun, Fei and Zhang, Wentao and Xie, Xu and Cui, Bin},
  journal={ACM Computing Surveys},
  volume={55},
  year={2022}
}

@article{DBLP:journals/corr/abs-2401-00713,
  author       = {Hourun Li and
                  Yusheng Zhao and
                  Zhengyang Mao and
                  Yifang Qin and
                  Zhiping Xiao and
                  Jiaqi Feng and
                  Yiyang Gu and
                  Wei Ju and
                  Xiao Luo and
                  Ming Zhang},
  title        = {A Survey on Graph Neural Networks in Intelligent Transportation Systems},
  journal      = {CoRR},
  year         = {2024},
}

@article{schur2018pitfalls,
  author       = {Oleksandr Shchur and
                  Maximilian Mumme and
                  Aleksandar Bojchevski and
                  Stephan G{\"{u}}nnemann},
  title        = {Pitfalls of Graph Neural Network Evaluation},
  journal      = {Relational Representation Learning Workshop (R2L 2018), NeurIPS 2018},
  year         = {2018},
 
}

@inproceedings{
bechler-speicher2025position,
title={Position: Graph Learning Will Lose Relevance Due To Poor Benchmarks},
author={Maya Bechler-Speicher and Ben Finkelshtein and Fabrizio Frasca and Luis M{\"u}ller and Jan T{\"o}nshoff and Antoine Siraudin and Viktor Zaverkin and Michael M. Bronstein and Mathias Niepert and Bryan Perozzi and Mikhail Galkin and Christopher Morris},
booktitle={Forty-second International Conference on Machine Learning Position Paper Track},
year={2025},
}

@inproceedings{DBLP:conf/iclr/0002Y21,
  author       = {Uri Alon and
                  Eran Yahav},
  title        = {On the Bottleneck of Graph Neural Networks and its Practical Implications},
  booktitle    = {9th International Conference on Learning Representations, 2021},
  year         = {2021}, 
}

@inproceedings{DBLP:conf/icml/Xing0LHS24,
  author       = {Yujie Xing and
                  Xiao Wang and
                  Yibo Li and
                  Hai Huang and
                  Chuan Shi},
  title        = {Less is More: on the Over-Globalizing Problem in Graph Transformers},
  booktitle    = {Forty-first International Conference on Machine Learning, {ICML} 2024},
  year         = {2024},
}

@inproceedings{10.1109/Allerton49937.2022.9929363,
author = {Banerjee, Pradeep Kr. and Karhadkar, Kedar and Wang, Yu Guang and Alon, Uri and Mont\'{u}far, Guido},
title = {Oversquashing in GNNs through the lens of information contraction and graph expansion},
year = {2022},
publisher = {IEEE Press},
booktitle = {2022 58th Annual Allerton Conference on Communication, Control, and Computing (Allerton)}
}

@article{DBLP:journals/corr/abs-2401-15444,
  author       = {Simi Job and
                  Xiaohui Tao and
                  Taotao Cai and
                  Lin Li and
                  Haoran Xie and
                  Jianming Yong},
  title        = {Towards Causal Classification: {A} Comprehensive Study on Graph Neural
                  Networks},
  journal      = {CoRR},
  year         = {2024}
}

@inproceedings{DBLP:conf/kdd/TangSWY09,
  author       = {Jie Tang and
                  Jimeng Sun and
                  Chi Wang and
                  Zi Yang},
 
  title        = {Social influence analysis in large-scale networks},
  booktitle    = {Proceedings of the 15th {ACM} {SIGKDD} International Conference on
                  Knowledge Discovery and Data Mining},
 
  year         = {2009},
 
}

\newpage
\appendix
\onecolumn

\section{Related Work}
\label{Ap:related_work}
\paragraph{The Role and Limitations of Feature-Graph Alignment in GNN Theory.}
Many works on the theoretical understanding of Graph Neural Networks (GNNs) adopt the contextual Stochastic Block Model (CSBM) \citep{DBLP:conf/nips/DeshpandeSMM18}, which couples the graph structure of the Stochastic Block Model (SBM) with node features generated conditionally on the community assignment. For instance, \cite{shi2024homophily} leverages CSBM to characterise the generalisation behaviour of GCNs analytically.
 \cite{DBLP:conf/iclr/WuCWJ23} uses CSBM to distinguish denoising and mixing effects in oversmoothing and evaluate how modifications like residual connections affect this trade-off. The work \cite{luan2023graph} introduces a variant of CSBM (CSBM-H) to formalise intra- and inter-class node distinguishability, showing that the effectiveness of GNNs is governed not just by homophily, but by a more nuanced interaction between neighbourhood similarity and class separability. While these works advance the theoretical understanding of GNNs, their reliance on the CSBMs limits their analysis to settings where node features and graph structure are well-aligned. \cite{DBLP:journals/tnn/QianERPB22} investigates this limitation and introduces a subspace alignment measure to capture the interplay between graph structure, features, and ground truth labels. However, it lacks theoretical justification. 
 \paragraph{Limitations of GNNs.}
Earlier research has examined failure scenarios such as oversmoothing, in which node representations become increasingly indistinguishable as the depth of the GNN increases. Theoretical analysis of linear GNNs in \cite{DBLP:conf/nips/Keriven22}, based on the generalisation error, demonstrates that while moderate smoothing steps can enhance performance by amplifying principal feature directions and community structures, deeper networks inevitably drive node features to uninformative constants. Other theoretical studies support this view. For instance, \cite{DBLP:conf/iclr/OonoS20} show that as the number of layers in a GCN increases, its ability to distinguish node features collapses exponentially---eventually retaining only information about connected components and node degrees. \cite{DBLP:journals/corr/abs-2008-09864} proves that various GCN variants converge to a cuboid as the number of layers tends to infinity. \cite{DBLP:journals/corr/abs-2303-10993} defines oversmoothing as exponential similarity convergence and shown empirically across architectures and datasets. 
Using the kernel equivalence of GCN, \cite{DBLP:journals/tmlr/SabanayagamEG23} shows the effect and rate at which oversmoothing happens with depth.
Another similar phenomenon is oversquashing, the distortion of long-range information flow caused by graph bottlenecks \citep{DBLP:conf/iclr/ToppingGC0B22,10.1109/Allerton49937.2022.9929363, pmlr-v202-black23a}. Oversmoothing and oversquashing are inherent limitations of GNNs, but can not explain the observed performance gap between different GNN architectures.
\paragraph{Expressivity of GNNs.}
The typical expressive power of Graph Neural Networks analysis is based on the Weisfeiler–Lehman (WL) graph isomorphism test \citep{DBLP:conf/aaai/0001RFHLRG19}, which determines whether two graphs are topologically equivalent, i.e., isomorphic.
\cite{DBLP:conf/iclr/XuHLJ19} establishes that standard message-passing neural networks (MPNNs)—a class of spatial graph neural networks that update node representations by iteratively aggregating information from neighboring nodes—cannot surpass the discriminative power of the 1-WL test, as their aggregation schemes produce identical node colorings for graphs deemed equivalent by the test. The Graph Isomorphism Network (GIN) \citep{DBLP:conf/iclr/XuHLJ19} architecture emerged as a provably maximally expressive MPNN under the 1-WL framework. However, despite its expressivity, GIN does not exhibit superior generalisation performance compared to other MPNNs \citep{DBLP:journals/corr/abs-2401-15444}, indicating that expressivity alone can not explain generalisation \citep{9892655}.
\paragraph{Generalisation Error Bounds.}
Several prior works have bounded the generalisation error of GNNs using tools from statistical learning theory. These include approaches based on algorithmic stability \citep{DBLP:conf/kdd/VermaZ19,DBLP:journals/ijon/ZhouW21}, Rademacher complexity \citep{DBLP:conf/nips/EsserVG21,DBLP:journals/corr/abs-2102-10234}, and PAC-Bayes theory \citep{DBLP:conf/iclr/LiaoUZ21}. However, these bounds are often restricted to a single architecture, loose, and do not capture the true generalisation. For instance, \cite{DBLP:conf/nips/EsserVG21} shows that applying classical tools such as VC dimension to GNNs can lead to vacuous bounds. \cite{DBLP:conf/icml/TangL23} further investigates the generalization properties of GNNs, showing how the network architecture impacts the generalization gap. Recent work \citep{DBLP:conf/iclr/LiG0025} have attempted to bridge the gap between generalisation and expressivity, proposing approaches like k-variance margin-based generalisation bounds. However, their bound is limited to MPNNs and does not fully explain observed performance variations across different models. Recent advances by \cite{DBLP:conf/iclr/RauchwergerJL25} provide distribution-agnostic generalisation bounds and universal approximation proofs for MPNNs. In contrast to these bounds, which are often architecture-specific and offer limited practical value, our work provides an exact characterisation of the generalisation error for a broad class of GNNs by adopting a signal processing perspective and offers clear insights into when and why GNNs can effectively leverage graph structure and node features.
\section{Frequency Response of GNNs}
\label{Ap:frequency_response}
As discussed in Section 2, we consider GNNs of the form \( S Z \theta \), where \( S \) denotes the propagation operator applied to the feature matrix \( Z \). In the case of a single support, we denote this operator by \( C \), i.e., \( S = C \).

More generally, GNNs may use multiple propagation supports \( \{C^{(j)}\}_{j=1}^m \), each defining a different pattern of feature propagation across the graph. In this case, the overall propagation operator becomes
\[
S = \begin{bmatrix}
C^{(1)} & C^{(2)} & \cdots & C^{(m)}
\end{bmatrix}.
\]

Each \( C^{(j)} \in \mathbb{R}^{n \times n} \) represents a \emph{convolution support}, which defines how node features are aggregated from neighboring nodes. Each \( C^{(j)} \) shares a common eigenbasis (e.g., derived from the graph Laplacian), and has a spectral response \( g_j(\Lambda) \). \cite{DBLP:conf/iclr/BalcilarRHGAH21} study the spectral response of individual convolution supports. However, when analysing the theoretical generalisation error in Theorem 3.1, the relevant quantity is \( S S^\top \) and it satisfies
\[
S S^\top = \sum_{j=1}^m C^{(j)} C^{(j)\top},
\]
and its spectral response becomes
\[
g^2(\Lambda) = \sum_{j=1}^m g_j^2(\Lambda).
\]

Consider \textbf{ChebNet} as an example of a GNN with multiple convolutional support: \\
\begin{align*}
    C^{1} &= I = UU^T, \quad C^{2} = \frac{2L}{\lambda_{\max}} - I = U(2\Lambda/\lambda_{max}-I)U^T, \quad C^{k} = 2C^{2}C^{k-1} - C^{k-2} \\
    S^{1} &= C^{1}Z W^{(0,1)} + C^{2}Z W^{(0,2)} + \dots = H_{\text{ChebNet}} W'^{T}
\end{align*}

where  
\begin{align*}
    H_{\text{ChebNet}} &:= \big[ C^{1}Z \quad C^{2}Z \quad \dots \big] \\
    &= \big[ Uf^{1/2}(\Lambda)\Phi^T \quad U(2\Lambda/\lambda_{max}-I)f^{1/2}(\Lambda)\Phi^T \quad \dots \big], \\
    W' &= \big[ W^{(0,1)} \quad W^{(0,2)} \quad \dots \big].
\end{align*}

\begin{align*}
H_{\text{ChebNet}}H_{\text{ChebNet}}^T &= U (f(\Lambda) +(2\Lambda/\lambda_{max}-I)^2f(\Lambda)+ \dots \big) U^T \\
&:= U \tilde{\Lambda} U^T
\end{align*}

\textbf{Remark on ChebNetII.} As seen in Table \ref{tab:gnn_summary} the variant referred to as ChebNetII in this work corresponds to ChebBase/s from \cite{DBLP:conf/nips/HeWW22}. It follows the same principle as ChebNetII in mitigating Runge’s phenomenon, as detailed in \cite{DBLP:conf/nips/HeWW22}.

Now consider the operator
\[
S := (D + I)^{-1/2}(A + I)(D + I)^{-1/2},
\]
which appears in several GNNs including PPNP and GPR-GNN. Its frequency response can be approximated as
\begin{equation}
g(\lambda) \approx 1 - \lambda \cdot \frac{\bar{p}}{\bar{p} + 1},
\label{Eq:fr}
\end{equation}
where \( \bar{p} \) denotes the average node degree in the graph. This expression assumes the graph is approximately regular, following the approximation used in \cite{DBLP:conf/iclr/BalcilarRHGAH21}.

We provide frequency responses of several GNN architectures, including PPNP, GPR-GNN, Highpass, and FAGCN in Table \ref{tab:gnn_summary}. For PPNP and GPR-GNN, the response of \((D + I)^{-1/2}(A + I)(D + I)^{-1/2}\) directly yields the corresponding expressions. In the case of Highpass and FAGCN, the frequency response follows from the convolution definitions by substituting the Laplacian \(L\) with its eigenvalue \(\lambda\).

\begin{itemize}
  \item \textbf{PPNP:} The convolution matrix is
  \[
  C = \alpha \left( I - (1 - \alpha) (D + I)^{-1/2} (A + I) (D + I)^{-1/2} \right)^{-1}.
  \]
  Substituting Equation \ref{Eq:fr} its frequency response becomes
  \[
  g(\lambda) \approx \alpha \left(1 - (1 - \alpha)\left(1 - \frac{\lambda \bar{p}}{\bar{p} + 1}\right)\right)^{-1}.
  \]

  \item \textbf{GPR-GNN:} The convolution matrix is
  \[
  C = \sum_{k=0}^K \gamma_k \left[ (D + I)^{-1/2}(A + I)(D + I)^{-1/2} \right]^k,
  \]
  which yields the frequency response
  \[
  g(\lambda) \approx \sum_{k=0}^K \gamma_k \left(1 - \frac{\lambda \bar{p}}{\bar{p} + 1} \right)^k.
  \]

  \item \textbf{Highpass:} The convolution operator is the graph Laplacian \(C = L\), so the frequency response is simply
  \[
  g(\lambda) = \lambda.
  \]

  \item \textbf{FAGCN:} The convolution matrix is a weighted combination of low- and high-pass terms,
  \[
  C = \alpha((1+\epsilon)I - L) + (1 - \alpha)((\epsilon - 1)I + L),
  \]
  and the corresponding frequency response is
  \[
  g(\lambda) = \alpha((1+\epsilon) - \lambda) + (1 - \alpha)((\epsilon - 1) + \lambda).
  \]

    \item \textbf{GAT:} GAT computes hidden representations by applying attention weights to neighbors:
      \[
      x_i' = \sigma\!\left( \sum_{j \in \mathcal{N}(i)} \alpha_{ij} W x_j \right),
      \]
      where \(\alpha_{ij}\) are learned attention coefficients and \(W\) is a weight matrix.  
      While GAT is not defined in the spectral domain, we interpret it as a learnable frequency response on the eigenvalues, i.e.,
      \( g(\lambda) = \inf_{\substack{g \\ g: \lambda \to \tilde{\lambda}}} R \).
     \item \textbf{Specformer:} Specformer is a spectral GNN that overcomes the limitations of classical scalar-to-scalar filters. It takes the full set of eigenvalues \(\Lambda\) as input and applies self-attention to produce a transformed spectrum:
       For each attention head \(m\), Specformer applies self-attention in the spectral domain to generate new eigenvalue representations:
\[
Z_m = \mathrm{Attention}(Q W_m^Q,\, K W_m^K,\, V W_m^V),
\qquad 
\lambda_m = \phi(Z_m W_\lambda),
\]
where \(Q,K,V\) are query, key, and value matrices,  
\(W_m^Q, W_m^K, W_m^V\) are learnable parameters, \(Z_m\) is the representation produced by the \(m\)-th head, and \(\phi\) is a nonlinearity (e.g., ReLU or Tanh).  

The resulting filtered eigenvalues \(\lambda_m \in \mathbb{R}^{n \times 1}\) are used to reconstruct learnable bases:
\[
S_m = U \,\mathrm{diag}(\lambda_m) U^\top,
\]
with \(U\) denoting the eigenvector matrix of the graph Laplacian.  
These bases are then concatenated and passed through a feed-forward network to produce the combined basis,
\[
\hat{S} = \mathrm{FFN}\!\big([I_n \,\|\, S_1 \,\|\, \cdots \,\|\, S_M]\big).
\]
        In our analysis, we model this as a set-to-set frequency response at the matrix level:
       \[g(\Lambda) = \inf_{\substack{g \\ g: \Lambda \to \tilde{\Lambda}}} R \]
        capturing global spectral patterns that element-wise filters cannot.
  
\end{itemize}

\begin{table}[t!]
\centering
\caption{Convolution operators \( C \) and frequency responses \( g(\lambda) \) of selected GNNs.}
\resizebox{\textwidth}{!}{%
\begin{tabular}{@{}>{\raggedright\arraybackslash}p{2.8cm} >{\raggedright\arraybackslash}p{8cm} >{\raggedright\arraybackslash}p{6.5cm}@{}}
\toprule
\textbf{GNN} & \textbf{Convolution matrix \( C \)} & \textbf{Frequency response \( g(\lambda) \)} \\
\midrule
MLP & \( C = I \) & \( g(\lambda) = 1 \) \\[3pt]

GCN \citep{DBLP:conf/iclr/KipfW17} & \( C = 2I - L \) & \( g(\lambda) = 2(1 - \lambda/2) \) \\[3pt]

GIN$^a$\ \citep{DBLP:conf/iclr/XuHLJ19} & \( C = A + (1 + \epsilon)I \) & \( g(\lambda) \approx \bar{p} \left( \frac{1 + \epsilon}{\bar{p}} + 1 - \lambda \right) \) \\[3pt]

PPNP$^a$\ \citep{DBLP:conf/iclr/KlicperaBG19} & \( C = \alpha(I - (1-\alpha)(D + I)^{-0.5}(A + I)(D + I)^{-0.5})^{-1} \) & \( g(\lambda) \approx \alpha \left(1 - (1 - \alpha)\left(1 - \frac{\lambda \bar{p}}{\bar{p} + 1}\right)\right)^{-1} \) \\[3pt]

GPR-GNN \citep{DBLP:conf/iclr/ChienP0M21} & \( C = \sum_{k=0}^K \gamma_k \left[(D + I)^{-0.5}(A + I)(D + I)^{-0.5}\right]^k \) & \( g(\lambda) \approx \sum_{k=0}^K \gamma_k \left(1 - \frac{\lambda \bar{p}}{\bar{p} + 1} \right)^k \) \\[3pt]

Highpass \citep{DBLP:conf/acml/ZhangL23} & \( C = L \) & \( g(\lambda) = \lambda \) \\[3pt]

$[\text{High} \mid \text{Low}]$ \cite{DBLP:conf/sspr/AnsarizadehTTR20} &
\begin{tabular}[t]{@{}l@{}}
     \( C^{(1)} = I - L/2 \) \\
     \( C^{(2)} = L/2 \)
\end{tabular}
&
\begin{tabular}[t]{@{}l@{}}
     \( g_1(\lambda) = 1 - \lambda/2 \) \\
     \( g_2(\lambda) = \lambda/2 \)
\end{tabular} \\[3pt]

FAGCN \citep{DBLP:conf/aaai/BoWSS21} & \( C = \alpha((1+\epsilon)I - L) + (1 - \alpha)((\epsilon - 1)I + L) \) & \( g(\lambda) = \alpha((1+\epsilon) - \lambda) + (1 - \alpha)((\epsilon - 1) + \lambda) \) \\[3pt]

GAT \citep{DBLP:conf/iclr/VelickovicCCRLB18} & trainable convolution & \( g(\lambda) = \inf_{\substack{g \\ g: \lambda \to \tilde{\lambda}}} R \) \\[3pt]

Specformer \citep{DBLP:conf/iclr/BoSWL23} & trainable convolution & \( g(\Lambda) = \inf_{\substack{g \\ g: \Lambda \to \tilde{\Lambda}}} R \) \\[3pt]

GraphSAGE \citep{DBLP:conf/nips/HamiltonYL17} &
\begin{tabular}[t]{@{}l@{}}
     \( C^{(1)} = I \) \\
     \( C^{(2)} = D^{-1}A \)
\end{tabular}
&
\begin{tabular}[t]{@{}l@{}}
     \( g_1(\lambda) = 1 \) \\
     \( g_2(\lambda) = 1 - \lambda \)
\end{tabular} \\[3pt]

CayleyNet$^b$ \citep{DBLP:journals/tsp/LevieMBB19} &
\begin{tabular}[t]{@{}l@{}}
     \( C^{(1)} = I \) \\
     \( C^{(2r)} = \mathrm{Re}(\rho(hL)^r) \) \\
     \( C^{(2r+1)} = \mathrm{Re}(i \rho(hL)^r) \)
\end{tabular}
&
\begin{tabular}[t]{@{}l@{}}
     \( g_1(\lambda) = 1 \) \\
     \( g_{2r}(\lambda) = \cos(r \theta(h \lambda)) \) \\
     \( g_{2r+1}(\lambda) = -\sin(r \theta(h \lambda)) \)
\end{tabular} \\[3pt]

ChebNet \citep{DBLP:conf/nips/DefferrardBV16} &
\begin{tabular}[t]{@{}l@{}}
    \( C^{(1)} = I \) \\
    \( C^{(2)} = 2L/\lambda_{\max} - I \) \\
    \( C^{(s)} = 2C^{(2)}C^{(s-1)} - C^{(s-2)} \)
\end{tabular}
&
\begin{tabular}[t]{@{}l@{}}
    \( g_1(\lambda) = 1 \) \\
    \( g_2(\lambda) = 2\lambda/\lambda_{\max} - 1 \) \\
    \( g_s(\lambda) = 2g_2(\lambda)g_{s-1}(\lambda) - g_{s-2}(\lambda) \)
\end{tabular} \\[3pt]

ChebNetII \citep{DBLP:conf/nips/HeWW22} &
\begin{tabular}[t]{@{}l@{}}
    \( C^{(1)} = I \) \\
    \( C^{(2)} = 2L/\lambda_{\max} - I \) \\
    \( C^{(s)} = \left(2C^{(2)}C^{(s-1)} - C^{(s-2)}\right)/s \)
\end{tabular}
&
\begin{tabular}[t]{@{}l@{}}
    \( g_1(\lambda) = 1 \) \\
    \( g_2(\lambda) = 2\lambda/\lambda_{\max} - 1 \) \\
    \( g_s(\lambda) = \left(2g_2(\lambda)g_{s-1}(\lambda) - g_{s-2}(\lambda)\right)/s \)
\end{tabular} \\[3pt]

\bottomrule
\end{tabular}%
}

$^a$ \( \bar{p} \) is the average node degree in the graph. 
$^b$ \( \rho(x) = (x - iI) / (x + iI) \), \( \theta(h\lambda) = \arg(\rho(h\lambda)) \)

\label{tab:gnn_summary}
\end{table}

\subsection{Discussion on Future Work}
Our analysis focuses on architectures that can be expressed in a spectral filter representation (see Section 1). This requirement enables us to derive exact generalisation error bounds and cover a broad range of architectures. However, it also leads to approximations for certain models such as GIN [1], PPNP [2], and GPR-GNN [3] due to the assumption of approximate graph regularity.

All the architectures we consider are summarised in Table \ref{tab:gnn_summary}. We also include idealised versions of attention-based models, as stated in Definition 3.2. The actual constructions of these models and our rationale for the reasoning behind Definition 3.2 are also provided in Appendix~\ref{Ap:frequency_response}. A concrete investigation of the attention mechanism is left for future work.

\newpage
\section{Proof of Theorem~\ref{thm:aniso}}
\label{Ap:theo_proof}
We derive the generalisation error under the framework introduced in Section 2, following a spectral approach similar to \cite{NIPS2001_d68a1827}. Under the parameter prior in Assumption 3.2, the ground truth model ("teacher") and the GNN model ("student") are represented as:
\( \displaystyle
\frac{1}{n}XX^T = U {\Lambda_{\star}}^{1/2} \Phi^T \theta^{\star} {\theta^{\star}}^T \Phi {\Lambda_{\star}}^{1/2}U^T,  \hspace{1mm}
\frac{1}{n}HH^T = U \tilde{\Lambda} U^T,
\)
where $\Lambda_{\star}, \tilde{\Lambda}$ are diagonal matrices encoding the signal spectra of the teacher and student models, respectively. 
The generalisation error dataset is given as:
\begin{align*}
R_H = \mathbb{E}_{V,\epsilon, \theta^\star} \left[\frac{1}{n}  \sum_{i=1}^{n} \left((H \hat{\theta})_i - (X \theta^\star)_i\right)^2 \right].
\end{align*}

We begin with the closed-form solution of the ridge regression estimator:
\[
\hat{\theta} = \left(H_{\text{train}}^\top H_{\text{train}} + \sigma^2 I\right)^{\dagger} H_{\text{train}}^\top y.
\]
Introducing a diagonal matrix \(I_{\text{train}} \in \mathbb{R}^{n \times n}\), which has ones on the diagonal entries corresponding to the training nodes and zeros elsewhere, since \(I_{\text{train}}^2=I_{\text{train}}\) we can express \(\hat{\theta}\) as:
\[
\hat{\theta}= \left(H_{\text{train}}^\top H_{\text{train}} + \sigma^2 I\right)^{\dagger} H^\top I_{\text{train}} y.
\]
Using the expression \(y = X \theta^\star + \epsilon\), this becomes:
\[
\hat{\theta} = \left(H_{\text{train}}^\top H_{\text{train}} + \sigma^2 I\right)^{\dagger} H^\top \left(I_{\text{train}} X \theta^\star + I_{\text{train}} \epsilon\right).
\]

Under Assumption~3.1, we can write  
\(
H_{\text{train}}^\top H_{\text{train}} = \frac{n_{\text{train}}}{n} H^\top H,
\)
and,  
\(H_{\text{train}}^\top X_{\text{train}} = \frac{n_{\text{train}}}{n} H^\top X .
\)
(Details on Assumption~3.1 are in Appendix~\ref{Ap:covariance})
By substituting these expressions and reformulating the ridge regression estimator in its dual form, we obtain:
\[
\hat{\theta} = H^\top \left(H H^\top + c_{1} I \right)^{\dagger} (X\theta^\star+\epsilon_{\text{train}}),
\]
where \(c_{1}=\frac{n\sigma^2}{n_{\text{train}}}\) and \(\epsilon_{\text{train}}:=\frac{n}{n_{\text{train}}}I_{\text{train}}\epsilon\).
We substitute this expression for \(\hat{\theta}\) and express the prediction error for each node \(i\) as:
\begin{align*}
\left((H \hat{\theta})_i - (X \theta^\star)_i\right)^2 
&= \left(e_i^T H H^\top \left(H H^\top + c_{1} I \right)^{\dagger} (X \theta^\star + \epsilon_{\text{train}}) - e_i^T X \theta^\star \right)^2 \\
&= \left(e_i^T H H^\top \left(H H^\top + c_{1} I \right)^{\dagger} X \theta^\star - e_i^T X \theta^\star \right)^2 \\
&\quad + \left(e_i^T H H^\top \left(H H^\top + c_{1} I \right)^{\dagger} \epsilon_{\text{train}} \right)^2 \\
&\quad + 2 \left(e_i^T H H^\top \left(H H^\top + c_{1} I \right)^{\dagger} X \theta^\star - e_i^T X \theta^\star \right) \left(e_i^T H H^\top \left(H H^\top + c_{1} I \right)^{\dagger} \epsilon_{\text{train}} \right).
\end{align*}

Taking the expectation over \(\epsilon\), and using \(\mathbb{E}[\epsilon] = 0\), the cross term vanishes, yielding the bias-variance decomposition:
\begin{align*}
R_H &= \underbrace{\mathbb{E} \left[  \sum_{i=1}^{n} \left(e_i^T H H^\top \left(H H^\top + c_{1} I \right)^{\dagger} X \theta^\star - e_i^T X \theta^\star \right)^2 \right]}_{\text{Bias}} \\
&\quad + \underbrace{\mathbb{E} \left[  \sum_{i=1}^{n} \left(e_i^T H H^\top \left(H H^\top + c_{1} I \right)^{\dagger} \epsilon_{\text{train}} \right)^2 \right]}_{\text{Variance}}.
\end{align*}

To simplify the variance term, observe that 
\(
\mathbb{E}_{\epsilon}[\epsilon \epsilon^\top] = \sigma^2 I,
\)
and 
\(
\mathbb{E}_{V}\!\left[\tfrac{n}{n_{\text{train}}} I_{\text{train}}\right] = I.
\)
Combining these identities with the definition of \(\epsilon_{\text{train}}\) yields
\(
\mathbb{E}[\epsilon_{\text{train}} \epsilon_{\text{train}}^\top] 
= \frac{n}{n_{\text{train}}}\,\sigma^2 I \;=\; c_{1} I.
\)
Hence the variance term can be expressed as

\begin{align*}
&\mathbb{E}\left[ \sum_{i=1}^{n} \left(e_i^T H H^\top \left(H H^\top + c_{1} I \right)^{\dagger} \epsilon_{\text{train}} \right)^2 \right] 
\\&= \sum_{i=1}^{n} \mathbb{E}\left[ \epsilon_{\text{train}}^\top \left(H H^\top + c_{1} I \right)^{\dagger} H H^\top e_i e_i^\top H H^\top \left(H H^\top + c_{1} I \right)^{\dagger} \epsilon_{\text{train}} \right] \\
&= \mathbb{E}\left[ \epsilon_{\text{train}}^\top \left(H H^\top + c_{1} I \right)^{\dagger} H H^\top \left( \sum_{i=1}^n e_i e_i^\top \right) H H^\top \left(H H^\top + c_{1} I \right)^{\dagger} \epsilon_{\text{train}}\right] \\
&= \mathbb{E}\left[ \epsilon_{\text{train}}^\top \left(H H^\top + c_{1} I \right)^{\dagger} H H^\top H H^\top \left(H H^\top + c_{1} I \right)^{\dagger} \epsilon_{\text{train}} \right] \\
&= \operatorname{tr} \left( \left(H H^\top + c_{1} I \right)^{\dagger} H H^\top H H^\top \left(H H^\top + c_{1} I \right)^{\dagger} \mathbb{E}_{\epsilon,V}[\epsilon_{\text{train}} \epsilon_{\text{train}}^\top] \right) \\
&= c_{1} \operatorname{tr} \left( H H^\top \left(H H^\top + c_{1} I \right)^{-2} H H^\top \right).
\end{align*}

Using the decomposition \(\frac{1}{n} H H^\top = U \tilde{\Lambda} U^\top \), we get:
\begin{align*}
\text{Variance} 
&= c_{1} \operatorname{tr} \left(n U \tilde{\Lambda} U^\top \left( nU \tilde{\Lambda} U^\top + c_{1} I \right)^{-2} nU \tilde{\Lambda} U^\top \right) \\
&=  c_{1} \operatorname{tr} \left( \tilde{\Lambda}^2 \left( \tilde{\Lambda} + \frac{c_{1}}{n} I \right)^{-2} \right),
\end{align*}
where the final expression follows from the orthonormality of \( U \) and the matrix inversion Theorem (Theorem 5.15) from \cite{schott2016matrix}.

To simplify the bias term, we first expand the square:
\begin{align*}
&\mathbb{E}\left[ \sum_{i=1}^{n} \left(e_i^T \left( H H^\top \left(H H^\top + c_{1} I \right)^{\dagger} - I \right) X \theta^\star \right)^2 \right] \\
&= \mathbb{E}\left[\operatorname{tr} \left( \left( H H^\top \left(H H^\top + c_{1} I \right)^{\dagger} - I \right) X \theta^\star {\theta^\star}^\top  X^\top \left( H H^\top \left(H H^\top + c_{1} I \right)^{\dagger} - I \right)^\top \right)\right]
\end{align*}

To proceed, we apply the decompositions \( \frac{1}{n}H H^\top = U \tilde{\Lambda} U^\top \) and \(\frac{1}{n} X \theta^\star {\theta^\star}^\top X^\top = U \Lambda_\star^{1/2} \Phi^\top \theta^\star {\theta^\star}^\top \Phi \Lambda_\star^{1/2} U^\top \).

\begin{align*}
\text{Bias} 
&= \mathbb{E}\left[\operatorname{tr} \left( \left( nU \tilde{\Lambda} \left(n\tilde{\Lambda} + c_{1} I \right)^{-1} U^\top - I \right)n U \Lambda_{\star}^{1/2} \Phi^\top \theta^\star {\theta^\star}^\top \Phi \Lambda_{\star}^{1/2} U^\top \left(n U \tilde{\Lambda} \left(n\tilde{\Lambda} + c_{1} I \right)^{-1} U^\top - I \right)^\top \right)\right] \\
&= \mathbb{E}\left[\operatorname{tr} \left( n\Lambda_{\star}^{1/2} \Phi^\top \theta^\star {\theta^\star}^\top \Phi \Lambda_{\star}^{1/2} \left( n\tilde{\Lambda} \left(n\tilde{\Lambda} + c_{1} I \right)^{-1} - I \right)^2 \right)\right] \\
&= \mathbb{E}\left[\operatorname{tr} \left( n\Lambda_{\star}^{1/2} \Phi^\top \theta^\star {\theta^\star}^\top \Phi \Lambda_{\star}^{1/2} \left( \tilde{\Lambda} \left(\tilde{\Lambda} + \frac{c_{1}}{n} I \right)^{-1} - I \right)^2 \right)\right] \\
&= \operatorname{tr} \left( n\Lambda_{\star}\mathbb{E}_{\theta^\star}[ \Phi^\top \theta^\star {\theta^\star}^\top \Phi] \left( I - 2 \tilde{\Lambda} \left(\tilde{\Lambda} + c I \right)^{-1} + \tilde{\Lambda}^2 \left(\tilde{\Lambda} + c I \right)^{-2} \right) \right),
\end{align*}
where the final expression follows from the orthonormality of \( U \) and the matrix inversion Theorem (Theorem 5.15) from \cite{schott2016matrix}.

Note that under the parameter prior assumption (Assumption 3.2), the expectation \(\mathbb{E}_{\theta^\star}[\Phi^\top \theta^\star {\theta^\star}^\top \Phi] \) appearing next to 
\( \Lambda_{\star} \)  acts as a scaling. More precisely, it scales each eigenvalue \(\lambda_i^\star 
\text{ by the factor } \mathbb{E}_{\theta^\star}[\langle \phi_i, \theta^\star \rangle^2]\).

Finally, note that both the bias and variance terms have been expressed as traces of diagonal matrices, which correspond to the sums of their diagonal entries. Combining these two terms yields the total generalization error:
\begin{align*}
R_H = \sum_{i=1}^{d} \left( \frac{\tilde{\lambda}_i \cdot c}{\tilde{\lambda}_i +c } - \left(\tilde{\lambda}_i - \lambda_i^\star \cdot \mathbb{E} \langle \phi_i, \theta^\star \rangle^2\right) \cdot \frac{c^2}{(\tilde{\lambda}_i + c )^2} \right).
\end{align*}

This matches the expression in Theorem 3.1 and concludes the proof.

\section{Remark for Population Covariance Matrix Assumption}
\label{Ap:covariance}
In this section, we first discuss the intuition behind Assumption 3.1 by showing that it holds in expectation (\ref{Ap:covariance_expectation}). We then present concentration bounds for empirical covariance matrices, providing further justification that the assumption is reasonable (\ref{Ap:covariance_concentration}). Finally, we derive the generalisation error without relying on the Population Covariance Matrix Assumption, using the concentration bounds introduced in \ref{Ap:covariance_concentration} (\ref{Ap:covariance_generror}).

\subsection{Assumption Holds in Expectation}
\label{Ap:covariance_expectation}
If training nodes $V_{\text{train}}$ are selected uniformly at random, then 
    \[
    \tfrac{1}{n} X^\top X = \mathbb{E}_{V_{\text{train}}} \left[ \tfrac{1}{n_{\text{train}}} X_{\text{train}}^\top X_{\text{train}} \right],
    \]
    and the same also holds for the other quantities $\tfrac{1}{n} X^\top H$ and $\tfrac{1}{n} H^\top H$.

\begin{proof}

Since $\frac{1}{\sqrt{n}}X = U \Lambda_{\star}^{1/2} \Phi^T$ and 
$X_{\text{train}} = I_{\text{train}} X 
$ 
(also note that $I_{\text{train}}^\top I_{\text{train}}=I_{\text{train}}$),
\begin{align}
\frac{1}{n} X^\top X 
&=  \Phi \Lambda_{\star}^{1/2} U^\top U \Lambda_{\star}^{1/2} \Phi^\top \\
&=  \Phi \Lambda_{\star} \Phi^\top \\
\frac{1}{n_{\text{train}}} X_{\text{train}}^\top X_{\text{train}}
&= \frac{n}{n_{\text{train}}} \Phi \Lambda_{\star}^{1/2} 
U^\top  I_{\text{train}} U \Lambda_{\star}^{1/2} \Phi^\top
\end{align}

Since 
\[
\frac{1}{n} \Phi \Lambda_{\star} \Phi^\top 
= \mathbb{E}_{V_{\text{train}}} \left[\Phi \Lambda_{\star}^{1/2} 
\frac{1}{n_{\text{train}}} U^\top  I_{\text{train}} U 
\Lambda_{\star}^{1/2} \Phi ^\top\right],
\]
and both sides can be multiplied by 
$(\Lambda_{\star}^{1/2})^{\dagger}\Phi^T$ and 
$\Phi(\Lambda_{\star}^{1/2} )^{\dagger}$ to obtain
\[
\frac{1}{n} I = \mathbb{E}_{V_{\text{train}}} \left[ \frac{1}{n_{\text{train}}} U^\top  I_{\text{train}} U \right].
\]

Since $H$ lies in the same subspace as $X$ 
($\frac{1}{\sqrt{n}}H = U \tilde{\Lambda}^{1/2}\Phi^T$),
\begin{align*}
\mathbb{E} \left[ \frac{1}{n_{\text{train}}} 
X_{\text{train}}^\top H_{\text{train}} \right]
&= \mathbb{E} \left[ \frac{n}{n_{\text{train}}} 
\Phi \Lambda_{\star}^{1/2} U^\top I_{\text{train}} U 
\tilde{\Lambda}^{1/2} \Phi^\top \right] \\
&=  \Phi \Lambda_{\star}^{1/2} 
\mathbb{E} \left[\frac{n}{n_{\text{train}}} 
U^\top I_{\text{train}} U \right] 
\tilde{\Lambda}^{1/2} \Phi^\top \tag{1} \\
&= \frac{1}{n}X^\top H
\end{align*}
\end{proof}

\subsection{Concentration Bounds for Empirical Covariance Matrices}
\label{Ap:covariance_concentration}
 In the regime where both $n$ and $n_{\text{train}}$ grow proportionally, that is, $n/d \to \infty$ and $n_{\text{train}} = \alpha n$ for a fixed $\alpha \in (0,1)$ ( which is often the case in many benchmark experiments), then standard results on the concentration of sample covariance matrices apply, ensuring that the empirical covariance from a random training subset closely approximates that of the full dataset. Specifically, under standard assumptions (e.g., sub-Gaussian rows), it is known that the sample covariance concentrates around its expectation (see \citep{koltchinskii2017concentration} and used for generalisation error analysis in \citep{DBLP:journals/corr/abs-1906-11300}), and that the difference
    \[
    \left\| \tfrac{1}{n} X^\top X - \tfrac{1}{n_{\text{train}}} X_{\text{train}}^\top X_{\text{train}} \right\| \to 0 \quad \text{as } n \to \infty.
    \]
The following lemma  characterizes the concentration behavior of empirical covariance matrices more formally. In the discussion below, with high probability (w.h.p.) refers to any bound that holds with probability $1-\delta$ where $\delta$ would vanish as $n\to\infty$.

\begin{lemma}[Concentration of empirical covariance 
\citep{koltchinskii2017concentration,DBLP:journals/corr/abs-1906-11300}]
Let $X \in \mathbb{R}^{n \times d}$ have sub-Gaussian rows. Then the empirical covariance matrices from the full dataset and the training subset concentrate around each other. In particular, the difference 
\[
\eta_1 := \left\| \frac{1}{n} X^\top X - \frac{1}{n_{\text{train}}} X_{\text{train}}^\top X_{\text{train}} \right\| \leq C\left\| \mathbb{E} \left[\frac{1}{n}X^TX \right]\right\|\left(\sqrt{\frac{d}{n}}+\sqrt{\frac{d}{n_{train}}}\right)=  O\left(\sqrt{\frac{d}{n_{train}}}\right)\qquad w.h.p.
\]
\label{lemma:con}
\end{lemma}

\begin{corollary}\label{cor:covariances}
\label{cor:con}
If Lemma \ref{lemma:con} holds, then the following concentration bounds also hold with high probability:
\[
\eta_2 := \left\| \frac{1}{n} H^\top H - \frac{1}{n_{\text{train}}} H_{\text{train}}^\top H_{\text{train}} \right\| \leq O\left(\sqrt{\frac{d}{n_{train}}}\right), \qquad w.h.p.  \tag{2.1}
\]
\[
\eta_3 := \left\| \frac{1}{n} H^\top X - \frac{1}{n_{\text{train}}} H_{\text{train}}^\top X_{\text{train}}  \right\| \leq O\left(\sqrt{\frac{d}{n_{train}}}\right) \qquad w.h.p.   \tag{2.2}
\]
The difference of the pseudo-inverses
\[
\eta_4 := \left\| \left( \frac{1}{n_{\text{train}}} H_{\text{train}}^\top H_{\text{train}} \right)^\dagger - \left( \frac{1}{n} H^\top H \right)^\dagger \right\|.
\]
can be bounded using (2.1) and Theorem 3.3 of \citep{stewart1977perturbation} to give 
\[
\eta_4 \leq 3 \cdot \eta_2 \cdot \max\left\{\left\|\left(\frac{1}{n_{\text{train}}} H_{\text{train}}^\top H_{\text{train}}\right)^\dagger\right\| , \left\|\left(\frac{1}{n} H^\top H\right)^\dagger\right\|\right) \quad w.h.p.
\]
assuming \(\left\|\left(\frac{1}{n} H^\top H\right)^\dagger\right\|\) and \(\left\|\left(\frac{1}{n_{\text{train}}} H_{\text{train}}^\top H_{\text{train}}\right)^\dagger\right\|\) are bounded.
\end{corollary}

\begin{proof}
We prove the bounds in Corollary~\ref{cor:covariances}.   
Recall the spectral decomposition $\frac{1}{\sqrt{n}}X = U \Lambda_{\star}^{1/2} \Phi^T$ and $X_{\text{train}} = I_{\text{train}} X = I_{\text{train}} \sqrt{n} U \Lambda_{\star}^{1/2}\Phi^ T$ (Also note that $I_{\text{train}}^\top I_{\text{train}}=I_{\text{train}}$)

It follows that
\begin{align}
\frac{1}{n} X^\top X 
&=  \Phi \Lambda_{\star}^{1/2} U^\top U \Lambda_{\star}^{1/2} \Phi^\top 
=  \Phi \Lambda_{\star} \Phi^\top, \\
\frac{1}{n_{\mathrm{train}}} X_{\mathrm{train}}^\top X_{\mathrm{train}} 
&= \frac{n}{n_{\mathrm{train}}} 
\Phi \Lambda_{\star}^{1/2} U^\top I_{\mathrm{train}} U \Lambda_{\star}^{1/2} \Phi^\top.
\end{align}

By the unitary invariance of the spectral norm and the orthogonality of $\Phi$, 
Lemma~\ref{lemma:con} yields
\[
\left\|
\frac{1}{n} X^\top X - \frac{1}{n_{\mathrm{train}}} X_{\mathrm{train}}^\top X_{\mathrm{train}}
\right\|
=
\left\|
 \Lambda_{\star} 
- \frac{n}{n_{\mathrm{train}}} \Lambda_{\star}^{1/2} U^\top I_{\mathrm{train}} U \Lambda_{\star}^{1/2}
\right\|
\leq \eta_1.
\]
Furthermore, by sub-multiplicativity of the spectral norm, we can bound
\[
\left\|
 I 
- \frac{n}{n_{\mathrm{train}}} U^\top I_{\mathrm{train}} U
\right\|
\leq 
\left\| \Lambda_{\star}^{-1/2} \right\| \,
\left\|
 \Lambda_{\star} 
- \frac{n}{n_{\mathrm{train}}} \Lambda_{\star}^{1/2} U^\top I_{\mathrm{train}} U \Lambda_{\star}^{1/2}
\right\| \,
\left\| \Lambda_{\star}^{-1/2} \right\|
\leq
 \frac{\eta_1}{\lambda_{\min}(\Lambda_{\star})},
\]
where $\lambda_{\min}(\Lambda_{\star})$ denotes smallest non-zero eigenvalue of $\Lambda_{\star}$.
It remains to bound the difference between the covariances of the hidden representations,
\begin{align}
\left\| \frac{1}{n} H^\top H - \frac{1}{n_{\mathrm{train}}} H_{\mathrm{train}}^\top H_{\mathrm{train}} \right\|
&= 
\left\|  \tilde{\Lambda} 
- \frac{n}{n_{\mathrm{train}}} \tilde{\Lambda}^{1/2} U^\top I_{\mathrm{train}} U \tilde{\Lambda}^{1/2} \right\| \nonumber \\
&\leq 
\left\| \tilde{\Lambda}^{1/2} \right\|
\left\|  I - \frac{n}{n_{\mathrm{train}}} U^\top I_{\mathrm{train}} U \right\|
\left\| \tilde{\Lambda}^{1/2} \right\| \nonumber \\
& \leq 
\frac{\lambda_{\max}(\tilde{\Lambda})}{\lambda_{\min}(\Lambda_\star)} \, \eta_1\nonumber \\
& = 
\frac{\lambda_{\max}(\frac{1}{n}H^TH)}{\lambda_{\min}(\frac{1}{n}X^TX)} \, \eta_1.
\nonumber 
\end{align}
And 
\begin{align}
\left\| \frac{1}{n} H^\top X - \frac{1}{n_{\mathrm{train}}} H_{\mathrm{train}}^\top X_{\mathrm{train}} \right\|
&= 
\left\|  \tilde{\Lambda}^{1/2}{\Lambda^{\star}}^{1/2}
- \frac{n}{n_{\mathrm{train}}} \tilde{\Lambda}^{1/2} U^\top I_{\mathrm{train}} U {\Lambda^{\star}}^{1/2} \right\| \nonumber \\
&\leq 
\left\| \tilde{\Lambda}^{1/2} \right\|
\left\|  I - \frac{n}{n_{\mathrm{train}}} U^\top I_{\mathrm{train}} U \right\|
\left\|{\Lambda^{\star}}^{1/2}\right\| \nonumber \\
& \leq 
\frac{\sqrt{\lambda_{\max}(\tilde{\Lambda})}\sqrt{\lambda_{\max}(\Lambda^{\star})}}{\lambda_{\min}(\Lambda_\star)} \, \eta_1
\nonumber \\
& =\frac{\lambda_{\max}(\frac{1}{n}H^TX)}{\lambda_{\min}(\frac{1}{n}X^TX)} \, \eta_1.
\nonumber 
\end{align}

\end{proof}

Consequently, the Assumption 2.1 holds in expectation, and the norm between these matrices is bounded which we use to provide \textbf{non-asymptotic bounds} that shows the gap between the correct generalisation error and the idealised one $R_H$ under the Assumption 3.1. 
\subsection{Generalisation Error Bounds without the Population Covariance Assumption}
\label{Ap:covariance_generror}
Building on Lemma \ref{lemma:con} and Corollary \ref{cor:con}, we provide the concentration bounds. 

Without the updated Assumption 3.1, the generalisation error should be
$$
R' := \frac{1}{n}\left\| H\theta_1 - X\theta^\star \right\|^2 \quad \text{where} \quad 
\theta_1 = \left( H_{\text{train}}^\top H_{\text{train}} +\sigma^2 I\right)^+ H_{\text{train}}^\top (X_{\text{train}} \theta^\star +\epsilon)
$$

We can rewrite $R'$ as $\frac{1}{n}\left\| H\theta_1 - H\theta_2 + H\theta_2 - X\theta^\star \right\|^2$, where $\theta_2 = \left( H^\top H +\frac{n\sigma^2}{n_{\mathrm{train}}}I  \right)^+ H^\top (X \theta^\star +\frac{n}{n_{\mathrm{train}}}I_{\mathrm{train}}\epsilon)$ 

Expanding the expression, we obtain 
$$
R' = 
\frac{1}{n}(\theta_1 - \theta_2)^\top H^\top H (\theta_1 - \theta_2) + \frac{2}{n}(\theta_1 - \theta_2)^\top H^\top (H\theta_2 - X\theta^\star) 
+\frac{1}{n} \left\| H\theta_2 - X\theta^\star \right\|^2.
$$

The last term ($\frac{1}{n}\left\| H\theta_2 - X\theta^\star \right\|^2$) is the generalisation error that we analysed in the paper and denoted by $R_H$. 

In the following, we show that the remaining terms are negligible given Corollary \ref{cor:con}, where $\eta_1, \eta_2, \eta_3$ are small and tend to zero with large $n$ (when $n_{train}=\alpha n$, for a fixed $\alpha\in (0,1)$). 
Define
\[
\psi_1 := \left( H_{\text{train}}^\top H_{\text{train}} +\sigma^2 I\right)^{\dagger} H_{\text{train}}^\top 
X_{\text{train}}  , 
\qquad
\psi_2 := \left( H^\top H +\frac{n\sigma^2}{n_{\mathrm{train}}}I\right)^{\dagger} H^\top 
 X  ,
\]
and set
\[
\gamma_1 := \left( H_{\text{train}}^\top H_{\text{train}} +\sigma^2 I \right)^{\dagger} H_{\text{train}}^\top,
\qquad
\gamma_2 := \left( H^\top H +\frac{n\sigma^2}{n_{\mathrm{train}}}I\right)^{\dagger} H^\top \frac{n}{n_{\mathrm{train}}}I_{\mathrm{train}}.
\]

The first term:
\begin{align*}
&\frac{1}{n}(\theta_1 - \theta_2)^\top H^\top H (\theta_1 - \theta_2)\leq \left\|\frac{1}{n} H^\top H \right\|  \left\| \psi_1 - \psi_2 \right\|^2 \left\| \theta^\star \right\|^2  + \frac{1}{n}\left\| H\gamma_1 - H\gamma_2 \right\|^2  \sigma^2
\end{align*}

\begin{align*}
&\left\|\frac{1}{n} H^\top H \right\|  \left\| \psi_1 - \psi_2 \right\|^2 \left\| \theta^\star \right\|^2 \\
&\leq \left\|\frac{1}{n} H^\top H \right\|  \left\| \left( \frac{1}{n_{\text{train}}}H_{\text{train}}^\top H_{\text{train}}+\frac{\sigma^2}{n_{train}} I \right)^\dagger \frac{1}{n_{\text{train}}}H_{\text{train}}^\top X_{\text{train}} - \left(\frac{1}{n} H^\top H +\frac{\sigma^2}{n_{\mathrm{train}}}I \right)^\dagger \frac{1}{n}H^\top X \right\|^2 \left\| \theta^\star \right\|^2\\
&\leq \left\|\frac{1}{n} H^\top H \right\|  \bigg(\left\| \left(\frac{1}{n_{\text{train}}} H_{\text{train}}^\top H_{\text{train}}+\frac{\sigma^2}{n_{train}} I \right)^\dagger \left(\frac{1}{n_{\text{train}}} H_{\text{train}}^\top X_{\text{train}} - \frac{1}{n}H^\top X \right) \right\| \\&+ \left\| \left( \left(\frac{1}{n_{\text{train}}} H_{\text{train}}^\top H_{\text{train}}+\frac{\sigma^2}{n_{train}} I \right)^\dagger - \left( \frac{1}{n}H^\top H +\frac{\sigma^2}{n_{\mathrm{train}}}I\right)^\dagger \right) \frac{1}{n} H^\top X \right\|\bigg)^2\left\| \theta^\star \right\|^2\\
&\leq \left\|\frac{1}{n} H^\top H \right\|  \left(\lambda_{\max}\left( (\frac{1}{n_{\text{train}}} H_{\text{train}}^\top H_{\text{train}}+\frac{\sigma^2}{n_{train}} I)^\dagger \right) \eta_3 + \eta_4 \, \lambda_{\max} \left( \frac{1}{n} H^\top X \right)\right)^2\left\| \theta^\star \right\|^2\\
&\leq O(\frac{1}{n}).
\end{align*}
assuming spectral norm of $\frac{1}{n}H^TH$, $(\frac{1}{n_{\text{train}}} H_{\text{train}}^\top H_{\text{train}}+\sigma^2 I)^\dagger$ and $ \frac{1}{n} H^\top X$  are bounded.

Next consider $\left\| H\gamma_1 - H\gamma_2 \right\|$:
\begin{align*}
&\left\| H\gamma_1 - H\gamma_2 \right\| \\
&=\left\| H(\frac{1}{n_{train}}H_{train}^TH_{train}+\frac{\sigma^2}{n_{train}}I)^{\dagger}H^T\frac{1}{n_{\mathrm{train}}}I_{\mathrm{train}} -H(\frac{1}{n}H^TH+\frac{\sigma^2}{n_{\mathrm{train}}}I)^{\dagger} H^T\frac{1}{n_{\mathrm{train}}}I_{\mathrm{train}}\right\|\\
&\leq \left\| H \left((\frac{1}{n_{train}}H_{train}^TH_{train}+\frac{\sigma^2}{n_{train}}I)^{\dagger} -(\frac{1}{n}H^TH+\frac{\sigma^2}{n_{\mathrm{train}}}I)^{\dagger}\right) H^T\frac{1}{n_{\mathrm{train}}}I_{\mathrm{train}}\right\| \\
& \leq \eta_4 \lambda_{max}(\frac{1}{n_{\mathrm{train}}}H_{\text{train}}^TH_{\text{train}})
\end{align*}

Putting them together, the first term $\frac{1}{n}(\theta_1 - \theta_2)^\top H^\top H (\theta_1 - \theta_2)\leq O(\frac{1}{n})+O(\frac{1}{n^2})=O(\frac{1}{n})$

For the second term, we observe that
$$
\begin{array}{l}
\displaystyle
\frac{2}{n} (\theta_1 - \theta_2)^\top H^\top (H \theta_2 - X \theta^\star) 
\quad \leq 2\left\| \theta_1 - \theta_2 \right\| \left\|\frac{1}{\sqrt{n}}H \right\| \left\| \frac{1}{\sqrt{n}} (H \theta_2 - X \theta^\star) \right\| \\[1.5ex]
\leq O(\frac{1}{\sqrt{n}}).
\end{array}
$$
since $\left\| \theta_1 - \theta_2 \right\| \leq O(\frac{1}{\sqrt{n}}) $,  $\left\|\frac{1}{\sqrt{n}}H \right\| \leq O(1)$ and $ \left\| \frac{1}{\sqrt{n}} (H \theta_2 - X \theta^\star) \right\| \leq O(1)$ with the same arguments.

Using above, we can bound the deviation between the generalisation errors w.h.p. as
$$
\begin{array}{l}
|R' - R_H|
= O\left(\frac{1}{\sqrt{n}}\right).
\end{array}
$$

We acknowledge that these bounds require assumptions that $d$ is small and the eigenvalues of the sample covariance matrices are bounded to prevent any terms from diverging. We keep the discussions in the paper in terms of $R_H$ instead of $R'$ since $R_H$ provides a more exact computation that can be analysed in specific settings.

\section{Proof of Corollary~\ref{gen-err}}
\label{Ap:cor41}
We start from the general expression for the generalisation error \(R_H\) given in Theorem 3.1 under the framework in Section 2 and Assumption 3.1. 

By substituting the isotropic parameter prior condition \( \mathbb{E}[\theta^\star \theta^{\star\top}] = I \) into the expression of \(R_H\), and noting that this implies \( \mathbb{E}\langle \phi_i, \theta^\star \rangle^2 = 1 \) due to the orthogonality of \(\Phi\), the summation terms simplify accordingly.

Specifically, the generalisation error decomposes into a sum over eigenvalues indexed by \(i\), and each summand depends only on the corresponding eigenvalue \(\tilde{\lambda}_i\) of the learned representation \(H H^\top\).

To analyse the behaviour of each summand, we consider the function
\[
r(\tilde{\lambda}_i) = 
\frac{\tilde{\lambda}_i \cdot c}{\tilde{\lambda}_i + c} 
- \left(\tilde{\lambda}_i - \lambda^\star_i\right) \frac{c^2}{\left(\tilde{\lambda}_i + c\right)^2},
\]
which represents the contribution of the \(i\)-th eigenvalue to \(R_H\). 

Our goal is to find the minimizer \(\tilde{\lambda}_i\) of this function for each \(i\). 

 Computing the derivative of \(r\) with respect to \(\tilde{\lambda}_i\), we get:
\begin{align*}
r'(\tilde{\lambda}_i) 
&= \frac{c(\tilde{\lambda}_i + c) - \tilde{\lambda}_i c}{(\tilde{\lambda}_i + c)^2}
- \left[ \frac{c^2(\tilde{\lambda}_i + c) - 2(\tilde{\lambda}_i - \lambda^\star_i)c^2}{(\tilde{\lambda}_i + c)^3} \right] \\
&= 2 \left( \tilde{\lambda}_i - \lambda^\star_i \right) \cdot \frac{c^2}{(\tilde{\lambda}_i + c)^3}.
\end{align*}

Setting \( r'(\tilde{\lambda}_i) = 0 \) to find critical points, we conclude that the minimum occurs at
\[
\tilde{\lambda}_i = \lambda^\star_i.
\]

This statement holds for all \(i\), establishing that the generalisation error is minimised when \(\tilde{\lambda}_i = \lambda^\star_i\) for each \(i\), which completes the proof of Corollary 4.1.

\section{Proof of Corollary~\ref{Cor:derivative}}
\label{Ap:cor42}
\label{Ap:Derivative}
Under the setting described in Section 2, let the filter \( g(\Lambda) \) be normalised to lie in \([0,1]\). We define \( a = \frac{\lambda_i}{2} \) and \( c = \frac{n\sigma^2}{n_{\text{train}}} \), where \( \lambda_i \) are the eigenvalues of the graph Laplacian and \( n_{\text{train}} \) is the number of training samples. Then, for a single-layer linear GCN, the derivative of the generalisation error with respect to the homophily parameter \( q \in [0,1] \) is given by:
\begin{align*}
&\frac{dR_{\text{GCN}}}{dq}(a=\frac{\lambda_{i}}{2}) =\\ 
&\frac{(-1 + 2a)c^2 \left(-c + a^2(6 - 23q) + a^4(8 - 18q) - q + a^5(-2 + 4q) + a(-1 + 8q) + a^3(-11 + 30q)\right)}
{\left(a + c + a^3(1 - 2q) + q - 4aq + a^2(-2 + 5q)\right)^3}.
\end{align*}

The derivative can also be obtained using the mathematical solver with the query:
\[
\text{D}\left[\frac{c(1-a)^2(q-(2q-1)a)}{(1-a)^2(q-(2q-1)a)+c} + \frac{(1-(1-a)^2)c^2(q-(2q-1)a)}{(1-a)^2(q-(2q-1)a)+c)^2}, q \right].
\]

Lastly, \(\displaystyle \frac{dR_{\mathrm{GCN}}}{dq} = \sum_{i}\frac{dR_{\mathrm{GCN}}(a=\frac{\lambda_i}{2})}{dq} \).
To get the results in Remark~\ref{Cor:Derivative}, we look at the sign of the term: \(\displaystyle \frac{dR_{\mathrm{GCN}}(a=\frac{\lambda_i}{2})}{dq} + \frac{dR_{\mathrm{GCN}}(a=\frac{\lambda_{max}-\lambda_i}{2})}{dq}\). Since the spectrum is symmetric, each term \( \frac{dR_{\mathrm{GCN}}}{dq}(a) \) is paired with \( \frac{dR_{\mathrm{GCN}}}{dq}(1 - a) \); if the sum of each such pair is negative for all \( 0.5 < a \leq 1 \), then the total derivative---being the sum of these negative pairs---is also negative. This is the statement of the first part of the Remark~\ref{Cor:Derivative} and can be obtained from the solver with the following query:
\[\text{Reduce[}\frac{dR_{\mathrm{GCN}}(a)}{dq} + \frac{dR_{\mathrm{GCN}}(1-a)}{dq} < 0 \text{ \&\& } 0.5 < a < 1\text{ \&\& } 0 < q < 1\text{ \&\& }\ c > 0.1 \text{]}
\]
Similarly second part of the Remark~\ref{Cor:Derivative} can be obtained from the solver with the queries like:
\begin{align*}
&\text{Reduce[}\frac{dR_{\mathrm{GCN}}(a)}{dq} + \frac{dR_{\mathrm{GCN}}(1-a)}{dq} > 0 \\& \text{ \&\& } c > 0 \text{ \&\& } a < 1 \text{ \&\& } q < 1 \text{ \&\& } q > 0 \text{ \&\& } a > 0.5 \text{ \&\& } c < 0.01\text{]}
\end{align*}
\(\text{Varying the intervals of } c \text{ and } a \text{ provides insight into how } 2 - \lambda_{\max} \text{ depends on } c, \) and how their relationship changes the risk under different levels of homophily.

\section{Proof of Corollary~\ref{cor:gat_specformer}}
\label{Ap:gat_specformer}

We begin by recalling the spectral interpretation of GAT and Specformer.
\paragraph{Spectral View of GAT and Specformer}
In the idealised setting, GAT applies a frequency response \( g: \lambda \mapsto \tilde{\lambda} \) to individual eigenvalues, while Specformer applies \( g: \Lambda \mapsto \tilde{\Lambda} \) to the full spectrum. Their generalisation errors are given by
\[
R_{\text{GAT}} = \inf_{g: \lambda \mapsto \tilde{\lambda}} R_H,
\quad
R_{\text{SF}} = \inf_{g: \Lambda \mapsto \tilde{\Lambda}} R_H.
\]

Let \( C = \left\{ \mathcal{I} \subseteq [d] \,|\, |\mathcal{I}| \geq 2,\ \lambda_i^\star = \lambda_j^\star \text{ and } \mathbb{E}\langle \theta^\star, \phi_i \rangle \neq \mathbb{E}\langle \theta^\star, \phi_j \rangle \forall i, j \in \mathcal{I} \right\}
\) be the collection of index sets of repeated eigenvalues in \(X\) with different alignment of \(\theta^\star\). 

For every \(\mathcal{I}  \in C\), GAT minimizes 
\begin{equation}
\label{Eq:gat}
|\mathcal{I}|\frac{\tilde\lambda_{i}c}{\tilde\lambda_{i}+c}-c^2\frac{|\mathcal{I}|\tilde\lambda_{i}-\sum_{j\in\mathcal{I}}\lambda^{\star}_j\mathbb{E}\langle \theta^\star, \phi_j \rangle^2}{(\tilde\lambda_{i}+c)^2} \text{, for all } i \in \mathcal{I} \end{equation}
over \(\tilde\lambda_{i}\). Hence \[
g_{\text{GAT}}(\lambda_{i})=\frac{1}{|\mathcal{I}|}\sum_{j\in\mathcal{I}}\lambda^{\star}_j\mathbb{E}\langle \theta^\star, \phi_j \rangle^2 \text{, for all } i \in \mathcal{I}.
\]
Substituting \(g_{\text{GAT}}(\lambda_{i})\) in Equation \ref{Eq:gat} gives the error of GAT for every \(\mathcal{I}  \in C\): \[\frac{ \left( \sum_{i \in \mathcal{I}} \lambda_i^\star  \mathbb{E}\langle \theta^\star, \phi_i \rangle^2 \right) c }
{ \left( \frac{1}{|\mathcal{I}|} \sum_{i \in \mathcal{I}} \lambda_i^\star \mathbb{E}\langle \theta^\star, \phi_i \rangle^2 \right) + c}.\]
Specformer achieves the minimal possible error for each eigenvalue, namely:
\[ \frac{ \lambda_i^\star  \mathbb{E}\langle \theta^\star, \phi_i \rangle^2  c }
{ \lambda_i^\star  \mathbb{E}\langle \theta^\star, \phi_i \rangle^2 + c }.\]

Note that for \(\mathcal{I} \notin C\), the error of GAT and Specformer does not differ.
Consequently, the generalisation error gap between GAT and Specformer arises only from the sets \( \mathcal{I} \in C \), and is given by
\[
 R_{\mathrm{GAT}} - R_{\mathrm{SF}} =
\sum_{\substack{\mathcal{I} \in C}} 
\left[
\frac{ \left( \sum_{i \in \mathcal{I}} \lambda_i^\star  \mathbb{E}\langle \theta^\star, \phi_i \rangle^2 \right)c }
{ \left( \frac{1}{|\mathcal{I}|} \sum_{i \in \mathcal{I}} \lambda_i^\star  \mathbb{E}\langle \theta^\star, \phi_i \rangle^2 \right) + c }
- \sum_{i \in \mathcal{I}} \frac{ \lambda_i^\star  \mathbb{E}\langle \theta^\star, \phi_i \rangle^2 c }
{ \lambda_i^\star  \mathbb{E}\langle \theta^\star, \phi_i \rangle^2 +c }
\right],
\]
which concludes the proof. 

\section{Experimental Details}
\label{Ap:Experiments}
All code and configurations used for the experiments are available at the following link:
\url{https://figshare.com/s/61f8fafb9469750f173e}
\paragraph{Hardware.} All experiments were conducted on a personal machine equipped with a 12th Gen Intel Core i7-12700H processor, 16 GB RAM, and an NVIDIA GeForce RTX 3050 GPU with 4 GB GDDR6 memory. The experiments were implemented in Python using PyTorch Geometric and executed under Windows 11.

\paragraph{Real-World Datasets.} 
We conduct experiments on six real-world node classification datasets commonly used in the literature: Cora \citep{10.5555/295240.295725}, Citeseer \citep{DBLP:conf/dl/GilesBL98}, Wikipedia (Wikipedia II from \cite{DBLP:journals/tnn/QianERPB22}), Squirrel \citep{DBLP:journals/compnet/RozemberczkiAS21} and Chameleon \citep{DBLP:journals/compnet/RozemberczkiAS21}. 
We load Chameleon and Squirrel using \texttt{Planetoid} classes from PyTorch Geometric and use the Geom-GCN versions. Other datasets, Cora, Citeseer and Wikipedia, are loaded using their respective open-source data files. All datasets are treated as undirected graphs. 

\subsection{Misalignment and Performance on Real-World Datasets}
We evaluate the performance of two architectures—graph convolution (denoted $SZ$) and concatenation ($[S\ Z]$)—on six standard node classification datasets. For each dataset, we compute the misalignment score defined in Definition~\ref{def:Misalignment} using $H = SZ$. The normalized misalignment is given by
\(\displaystyle
\frac{\mathrm{Tr}((I - P_H)XX^\top)}{\mathrm{Tr}(XX^\top)}
\), where $X$ is the one-hot encoded label matrix.
\textbf{The results are presented in Figure 2 of the main paper.}
We implement graph convolution (GCN) and the concatenation model with one layer of size $64$, ReLU activation, dropout rate $0.5$, and the Adam optimizer with learning rate $0.05$ and no weight decay. Each model is trained for up to $1000$ epochs with early stopping based on validation loss and a patience of $200$ epochs.

All results are averaged over $10$ random data splits. For each split, we sample $5\%$ of the nodes uniformly at random for training, and divide the remaining nodes equally between validation and test sets. The splits are generated independently of the graph structure, features, or labels. We use feature matrices without self-loops or feature normalisation to ensure that the representation space remains unaffected by preprocessing, allowing a faithful evaluation of misalignment.

\subsection{Theoretical Error Across Architectures vs. Homophily Parameter}
We analyse the theoretical error in Theorem 3.1 of several GNN architectures under varying homophily, parameterised by \( q \in [0, 1] \) using Equation 6. The error is computed with \( n = 100 \) and  \( \frac{\sigma^2 n}{n_{\text{train}}}  = 1 \). We model spectral properties using uniformly spaced \(n\) eigenvalues \( \lambda_i \in [0, 2] \). The frequency responses are normalised between 0 and 1. The plots for ChebNet I and ChebNet II, PPNP and GPR-GNN are presented in \textbf{Figure~3 (right plots in (a) and (b), respectively,)} in the main paper and for the other architectures in \textbf{Figure~\ref{fig:all_models_risk_analysis}} in Appendix~\ref{Ap:homo_exp} .

\subsection{Theoretical Error of GCN vs. Homophily Parameter}

We analyse the behaviour of the theoretical GCN error as a function of the homophily parameter \( q \) on the \textsc{Cora} and \textsc{Squirrel} datasets. \textbf{This experiment corresponds to Figure 4 in the main paper and \textbf{Figure~\ref{fig:gcn-risk-appendix}} in the appendix.}
We compute the eigenvalues of the normalised graph Laplacian. The theoretical error is then evaluated using Theorem 3.1 and Equation 6, which depends on the spectrum of the Laplacian, the homophily parameter \( q \in [0, 1] \), and a noise-to-signal parameter \( c = \frac{n\sigma^2}{n_{\text{train}}} \).
In \textbf{Figure 4}, we fix \( c = 0.01 \) and plot \( R_{\text{GCN}} \) for both datasets over a grid of 100 equally spaced values of \( q \). To illustrate the effect of \(c\), we include additional plots for \( c \in \{0.1, 0.01, 0.001, 0.0001\} \) in Figure~\ref{fig:gcn-risk-appendix}.

\subsection{Generalisation Performance under Heterophily}
\label{sec:heterophilic-perturbations}
\textbf{Figure 5 in the main paper} studies how the test error of GCN is affected when heterophilic edges are introduced to the Cora graph. We generate a new adjacency matrix for each setting by adding a fixed number of heterophilic edges to the original graph. A new edge $(i,j)$ is heterophilic if the labels $y_i$ and $y_j$ differ. These edges are sampled in a label-aware manner: for a node $i$ with label $y_i$, we draw a target label $c \ne y_i$ from the average label distribution in the neighbourhoods of nodes with label $y_i$, excluding self-class neighbors, and then sample a node $j$ with label $c$ uniformly at random. This ensures that the new edges connect different classes in a way that reflects the natural class distribution in the dataset.

To measure the impact of the heterophilic edges, we compute the homophily ratio as the fraction of edges that connect nodes with identical labels, as standard in the literature~\citep{DBLP:conf/nips/ZhuYZHAK20}. We compute the GCN accuracy over 50 independent runs for each new adjacency matrix and report the mean and standard deviation. We also track the maximum eigenvalue of the symmetric normalised Laplacian.

We use a 2-layer GCN with hidden dimension $64$, dropout rate $0.5$, learning rate $0.05$, and no weight decay. Training is performed with early stopping on a validation set using patience of $100$ epochs and a maximum of $1000$ epochs. The perturbation level varies from $0$ to $15\,000$ additional edges in increments of $1500$.

This experiment provides empirical support for the theoretical link between the spectral properties of the graph and GCN performance, and shows that GCNs may still generalise well under heterophily if the spectral conditions are favourable.

\subsection{Performance of GAT and Specformer under Repeated Eigenvalues in Synthetic Data}
 \label{sec:repeated-eigenvalues}
To empirically validate Corollary 4.3, we compare the performance of Specformer and GAT on a synthetic graph sequence designed to exhibit increasing eigenvalue multiplicity. The graphs are built by appending disjoint cycle graphs, where all nodes in each cycle are assigned the same class label. Each added block increases the number of distinct classes. Since cycle graphs have distinct Laplacian eigenvalues, replicating the same block structure across the graph increases the multiplicity of each eigenvalue. As shown in \textbf{Figure 6 in the main paper}, Specformer consistently achieves high accuracy, while the accuracy of GAT degrades as the number of repeated eigenvalues increases.
Each synthetic graph is constructed from $n_\text{blocks} \in \{1, \dots, 9\}$ disjoint cycle graphs of size $80$ nodes, resulting in a total of $80 \cdot n_\text{blocks}$ nodes. This block size was chosen to ensure a sufficiently large graph for training and evaluation. The adjacency matrix is constructed using the block-diagonal composition of the individual cycle graphs. Each node is assigned a $16$-dimensional feature vector sampled i.i.d. from a standard Gaussian distribution. This choice of random features prevents the task from becoming trivial due to informative features. Labels are assigned so that all nodes within a cycle graph share a single class, with one unique class per block.
The nodes are split randomly into training, validation, and test sets for each graph using a $60$-$20$-$20$ split. The same split is used for both models to ensure comparability. We train Specformer and GAT with two layers, hidden size $64$ and early stopping based on validation loss with patience of $200$ epochs and a maximum of $1000$ epochs. The Adam optimizer is used with a learning rate of $0.001$ and weight decay of $0.0005$. Each experiment is repeated for $50$ runs.
For each configuration, we compute the number of repeated eigenvalues in the adjacency matrix and report the mean and standard deviation of the test accuracy across runs.

\section{Additional Experiments of Section 4.2}
\label{Ap:homo_exp}
This appendix complements Section 4.2 by providing additional plots that support our theoretical analysis for a broad range of GNNs.

Figure \ref{fig:all_models_risk_analysis} shows how risk varies with respect to the homophily parameter \( q \). For this analysis, we consider a uniform eigenvalue distribution and  \(c=1\). Notably, the risk of GCNs decreases with increasing q, reflecting their low-pass nature, while high-pass filters exhibit the opposite trend, performing better in the heterophilic regime. For GIN \citep{DBLP:conf/iclr/XuHLJ19}, we observe that the risk depends on the hyperparameter \( \epsilon \), allowing the model to switch between high-pass and low-pass behaviors. ChebNet \citep{DBLP:conf/nips/DefferrardBV16}, as well as a concatenation of high-pass and low-pass filters (denoted \( [ \text{High} \hspace{2mm} \text{Low}] \) in \cite{DBLP:conf/sspr/AnsarizadehTTR20}), can handle strongly homophilic or heterophilic graphs but struggle in the intermediate range. As discussed earlier, GPR-GNN outperforms PPNP in handling heterophilic graphs, and it also exhibits a lower risk in general. Investigating the generalisation error of GNNs from a filtering perspective allows us to characterize the suitability of a particular GNN for homophily or heterophily. GNNs that have low-pass filter characteristics are better suited for homophilic graphs, while they struggle with heterophilic graphs, where high-frequency components dominate. This analysis provides valuable insights into the limitations of various GNN architectures in relation to the spectral properties of the graph data.

\begin{figure}[]
\centering
\includegraphics[width=\textwidth]{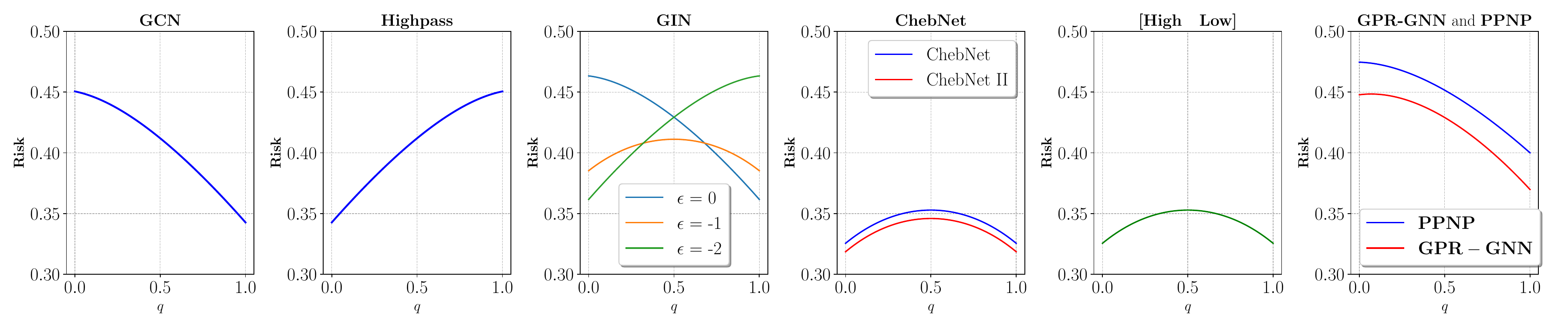}
\caption{Theoretical error over \( q \), averaged over eigenvalues, for different models including GCN, Highpass, GIN, ChebNet, ChebNet II, \([\text{High}\ \text{Low}]\), PPNP, and GPR-GNN.}
\label{fig:all_models_risk_analysis}
\end{figure}

\begin{figure}[]
\centering
\includegraphics[width=0.6\textwidth]{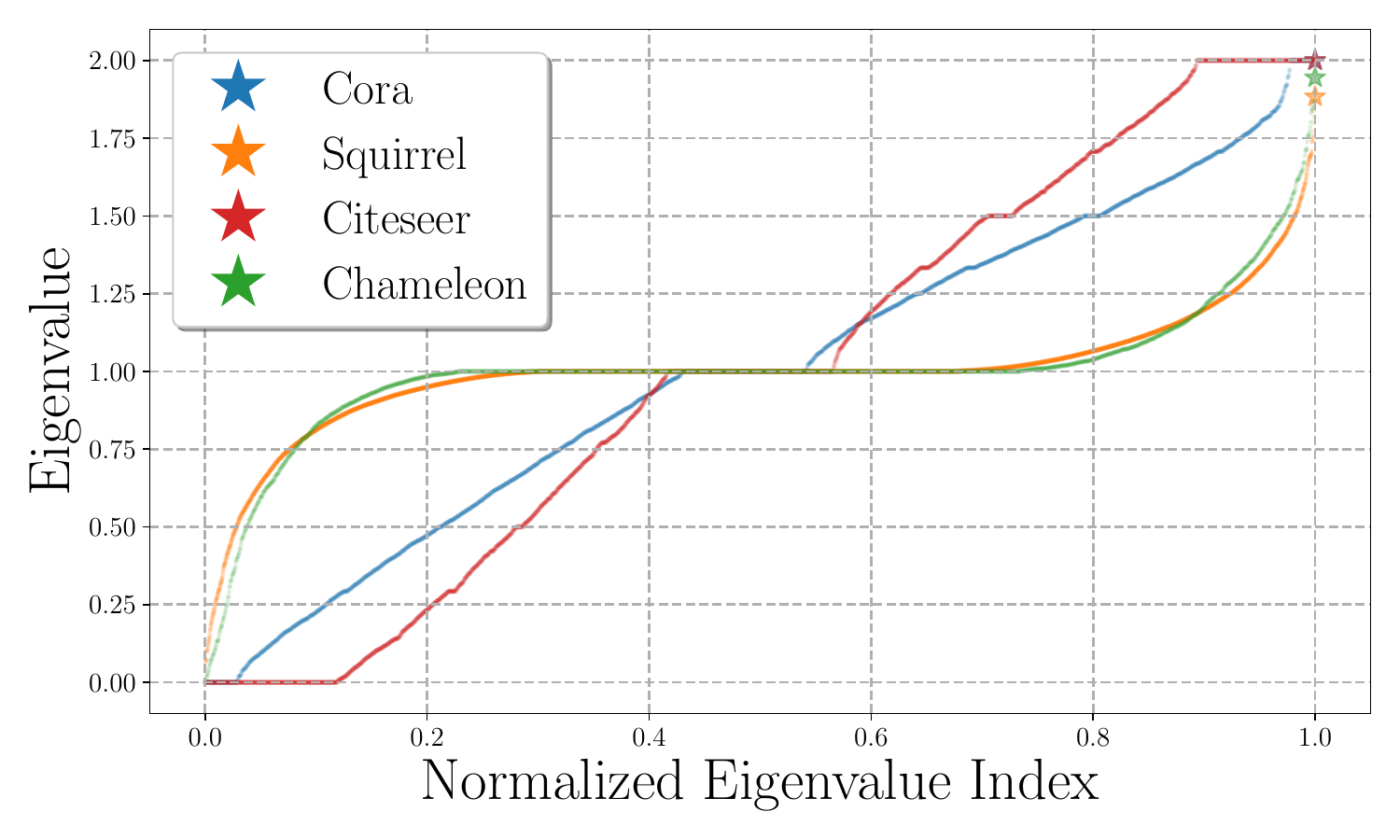}
\caption{Sorted eigenvalues of the normalised Laplacian $L$ for Cora, Squirrel, Citeseer, and Chameleon. The x-axis shows the normalised index of the eigenvalues.}
\label{fig:appendix_eigenvalues}
\end{figure}

Next, we provide additional plots of the spectral properties of real-world datasets and their influence on the generalization behavior of GCNs. Specifically, we analyse the eigenvalues of the Laplacian and visualise the theoretical error of GCN as a function of the homophily parameter \(q\) under different noise-to-label ratios \(\frac{{n\sigma^2}}{{n_{{\text{train}}}}}\). Figure \ref{fig:appendix_eigenvalues} shows that the graphs Cora, Citeseer, Squirrel and Chameleon all have approximately symmetric spectrum. The largest eigenvalue of the normalised Laplacian is exactly 2 for Cora and Citeseer, while it is strictly less than 2 for Squirrel and Chameleon. As stated in Remark~\ref{Cor:Derivative}, Figure \ref{fig:gcn-risk-appendix} shows that the GCN error on Cora (where \(\lambda_{\max} = 2\)) decreases with increasing \(q\). On Squirrel (where \(\lambda_{\text{max}} < 2\)), the behavior depends on the value of \(c\): for large values, the error decreases with \(q\); for small values, it increases; and for intermediate values, the trend is non-monotonic.

\begin{figure}[]
\centering
\includegraphics[width=\linewidth]{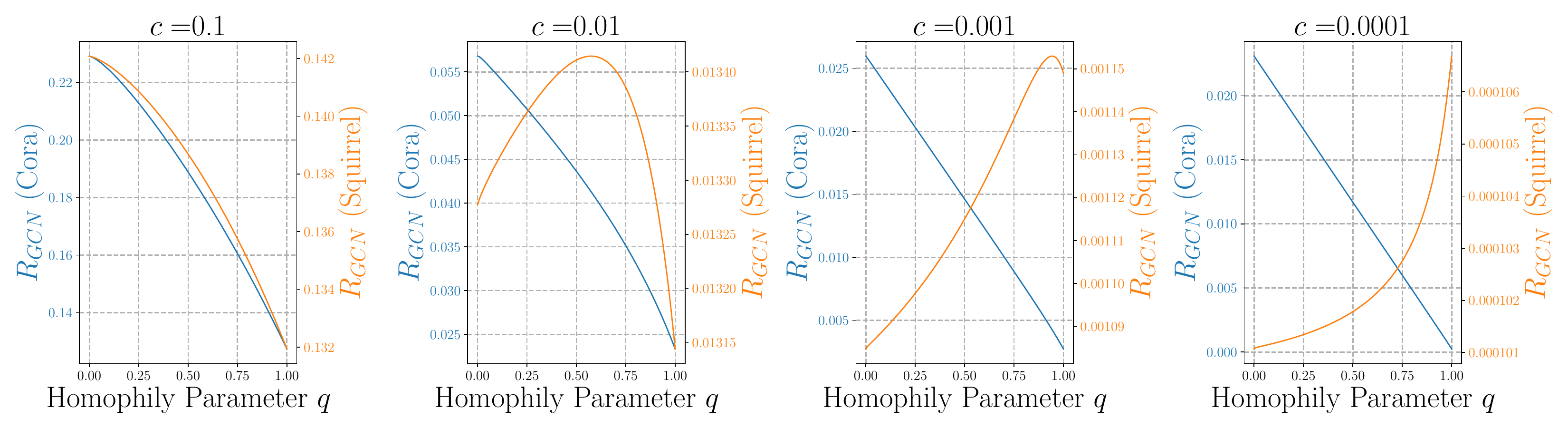}
\caption{GCN error curves for Cora and Squirrel, averaged over eigenvalues, plotted against the homophily parameter \( q \), under various noise-to-label ratios \( c \in \{1, 0.1, 0.01, 0.001\} \).}
\label{fig:gcn-risk-appendix}
\end{figure}

\section{Oversmoothing from a Filtering Perspective}
\label{Ap:Oversmoothing}
This section explores the role of depth and skip connections from a filtering perspective. The previous sections focus on \(l=1\). However, the presented theory can be extended to models with multiple layers and those incorporating skip connections. 
When considering deeper models, we observe a natural trend in the frequency response: as the number of layers increases, the filter becomes smoother. This smoothing effect indicates the oversmoothing phenomenon \citep{10.5555/3504035.3504468} commonly encountered in deeper GNNs, where information from distant nodes is excessively averaged, potentially leading to a loss of local structure.
To gain insight into how the network's depth affects its filtering behaviour, we examine the frequency response of GCN with increasing layer depth and the inclusion of skip connections. We see that as the number of layers increases, the filters suppress more high-frequency components of the graph signal (Figure \ref{fig:freq_response_gcn} Left). Skip connections, as expected, mitigate the oversmoothing effect by allowing direct connections between the input and output of each layer (Figure \ref{fig:freq_response_gcn} Right). 

\begin{figure}[]
\centering
\includegraphics[width=0.8\textwidth]{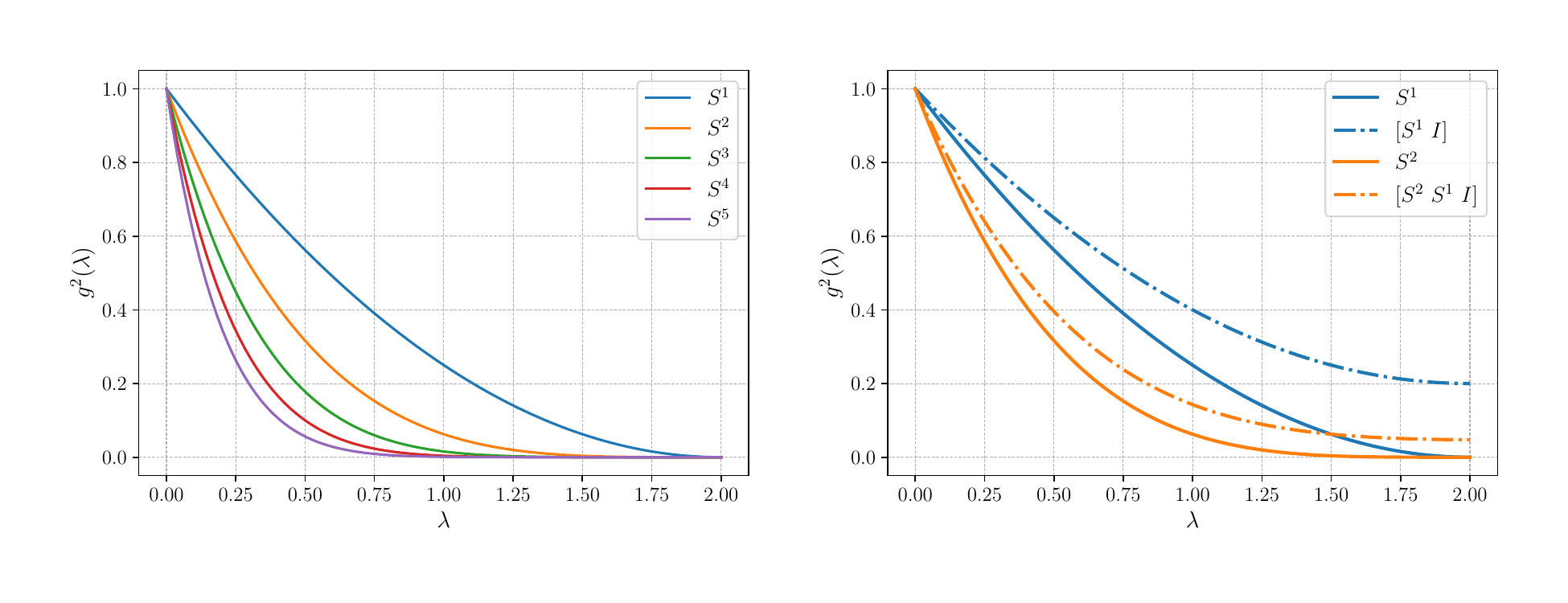}
\caption{Frequency response of Graph Convolutional Networks (GCNs) with varying depth and the effect of skip connections. \textbf{Left:} Frequency response for increasing number of layers, illustrating how deeper networks oversmooth by suppressing high-frequency components. \textbf{Right:} Frequency response of GCNs with skip connections, which mitigates oversmoothing.}
\label{fig:freq_response_gcn}
\end{figure}

\section{Practical Implications of the Theoretical Results}
The primary focus of our work is to provide an exact characterisation of the generalisation error of GNNs. This analysis reveals several insights, one significant implication being the presence of benchmark biases. 

While creating better benchmarks is outside the scope of this paper, our framework offers concrete tools to assess the interaction between graph structure and features. Beyond diagnosing biases, our exact generalisation error analysis provides principled criteria for model selection.

\begin{itemize}
    \item \textbf{Bias 1: Alignment Between Features and Graph Structure}  The generalisation error reveals this bias through the misalignment score (Lemma 3.1), which quantifies the extent to which the target space is not represented in the GNN embedding and can be computed for any dataset and model, and provides valuable information for model selection. It shows models like GCN are susceptible to this alignment, and that datasets with low alignment may inherently disadvantage such models. Figure 2 shows that when misalignment is high, even a simple concatenation (also introduced in \citep{chen2022demystifyinggraphconvolutionsimple}) outperforms GCN. It also reveals that most benchmark datasets (Cora, Citeseer, Chameleon, Squirrel) exhibit low misalignment and are implicitly favourable to GCN. In contrast, the Wikipedia dataset \citep{DBLP:journals/corr/abs-1905-12921}—which is less commonly used in the literature—shows high misalignment. As a result, the proposed misalignment score serves as a diagnostic tool to assess whether a dataset inherently favours multiplicative models, such as GCN and can be used to validate both existing and newly proposed benchmark datasets. Recognising this bias is valuable in itself, as it allows researchers to make informed choices when evaluating or designing models. Practitioners developing new models can use this score to pick an appropriate model for the dataset. For new benchmarks, this score provides a better understanding. This complements recent calls in the literature \citep{DBLP:journals/corr/abs-2502-14546,DBLP:journals/corr/abs-2502-02379} by making this benchmark bias measurable and interpretable within a principled framework.
    \item \textbf{Bias 2: Homophily Versus Heterophily}:
    Several prior works have already identified the bias of existing benchmarks to homophilic datasets \citep{DBLP:journals/corr/abs-2104-01404}.
    Our theory also provides a formal understanding of how GNN performance is influenced by the spectral characteristics of the graph and features, such as homophily and heterophily (Section 3B). By analysing the frequency response and how the generalisation error varies with the homophily parameter, our theory provides concrete guidance on which GNN architectures are better suited for different types of graphs. This relationship is detailed in Figure~\ref{fig:freq_response}. Our analysis also reveals an intriguing behaviour of GCN under extreme heterophily: when the graph exhibits spectral symmetry, the generalisation error of GCN can decrease with increasing heterophily, depending on the maximum eigenvalue of the Laplacian (see Remark 3.2, Fig. 3, Fig. 4). This allows practitioners to evaluate the expected performance of a GNN on a specific dataset based on its spectral profile, even before training. Furthermore, new architectures that fall within our signal processing framework can also be evaluated theoretically.
\end{itemize}

These can guide practitioners in evaluating the suitability of GNN architectures and in interpreting empirical performance more reliably.

\end{document}